\documentclass[twoside]{article}
\usepackage[accepted]{aistats2023}

\usepackage[round]{natbib}

\usepackage{booktabs}
\usepackage{enumitem}

\usepackage{mydefs}
\usepackage{notation}

\begin{document}

\twocolumn[
    \aistatstitle{Boosted Off-Policy Learning}
    \aistatsauthor{Ben London \And Levi Lu \And Ted Sandler \And Thorsten Joachims}
    \aistatsaddress{Amazon \And Amazon \And Groundlight.ai \And Amazon}
]


\begin{abstract}
We propose the first boosting algorithm for off-policy learning from logged bandit feedback. Unlike existing boosting methods for supervised learning, our algorithm directly optimizes an estimate of the policy's expected reward. We analyze this algorithm and prove that the excess empirical risk decreases (possibly exponentially fast) with each round of boosting, provided a ``weak" learning condition is satisfied by the base learner. We further show how to reduce the base learner to supervised learning, which opens up a broad range of readily available base learners with practical benefits, such as decision trees. Experiments indicate that our algorithm inherits many desirable properties of tree-based boosting algorithms (e.g., robustness to feature scaling and hyperparameter tuning), and that it can outperform off-policy learning with deep neural networks as well as methods that simply regress on the observed rewards.
\end{abstract}


\section{INTRODUCTION}
\label{sec:intro}

Boosting algorithms \citep{schapire:book12} have been a go-to approach for supervised learning problems, where they have been highly successful across a wide range of applications. Not only have they been the winning method for many Kaggle competitions \citep{chen:kdd16}, a recent empirical study  \citep{yang:kdd20} shows that gradient-boosted trees \citep{friedman:techreport99,mason:nips99} are the most successful model across the OpenML \citep{vanschoren:kdde13,feurer:corr19} tasks, providing the best classification accuracy on $38.60\%$ of the datasets. In contrast, their closest competitor, multi-layer perceptrons, achieve top performance on only $20.93\%$ of the datasets. Beyond pure predictive accuracy, boosted ensembles are also considered robust to feature scaling and hyperparameter tuning, and generally easier to train than neural networks, which often makes them a more practical choice for real-world applications \citep{grinsztajn:corr22}.

While boosting has enjoyed much success in supervised learning, a growing number of applications---such as search, recommendation and display advertising---fall outside the realm of supervised learning. Arguably, these applications are better formulated as contextual bandit problems, since the feedback (i.e., supervision) one observes depends on the actions that are taken. In practice, training often happens offline, using logged bandit feedback. For example, search ranking algorithms are usually trained on logged click data rather than in real time on live web traffic. This offline setting necessitates an \emph{off-policy} approach to learning, since the data used for training is collected by a different policy than the one being trained. To date, many off-policy learning algorithms have been proposed \citep{strehl:nips10,dudik:icml11,bottou:jmlr13,swaminathan:jmlr15,swaminathan:nips15,joachims:iclr18,wu:icml18,ma:aistats19,kallus:jasa19,london:icml19,chen:nips19,faury:aaai20,jeunen:kdd20}, yet none have explored a boosting approach.

In this paper, we derive the first boosting algorithm designed specifically for off-policy learning of contextual bandit policies. The algorithm, which we call \emph{BOPL} (for \emph{boosted off-policy learning}), sequentially constructs an \emph{ensemble policy}---comprised of a linear combination of predictors (e.g., decision trees)---by directly optimizing an estimate of the policy's expected reward---which is, after all, the quantity of interest when the policy is deployed. This \emph{policy optimization} approach stands in contrast to methods that indirectly derive a policy by regressing on the observed rewards, since better reward prediction (in terms of squared error) does not necessarily result in a better policy \citep{beygelzimer:kdd09}.

The BOPL algorithm is built upon a rigorous theoretical foundation. We first prove that its learning objective is \emph{smooth} (i.e., has a Lipschitz gradient), and then use this property to derive a gradient boosting algorithm. Our smoothness analysis has the advantage of giving us a closed-form expression for the ensemble weight of any predictor, and an upper bound on the learning objective at each round of boosting. To address optimization issues, we derive a variant of BOPL, called \emph{BOPL-S}, that optimizes a surrogate objective, which upper-bounds BOPL's objective and can be convex. Unlike some existing off-policy methods \citep{dudik:icml11,bottou:jmlr13,swaminathan:jmlr15,swaminathan:nips15,joachims:iclr18,london:icml19}, which require models to be differentiable in a fixed and enumerable set of parameters, BOPL only requires a \emph{base learner} that, at each round, trains a predictor to approximate the gradient with respect to the current ensemble. We show how this base learning objective can be reduced to a weighted regression or binary classification problem, which can be solved by off-the-shelf supervised learning algorithms. Moreover, we prove an upper bound on the excess empirical risk (that is, the empirical suboptimality) that decreases with each round of boosting as long as the base learner can produce a nontrivial predictor (akin to a ``weak" learning condition). Under certain conditions, the convergence rate can be exponentially decreasing in the number of rounds. When coupled with concentration and uniform convergence bounds, our excess empirical risk bound yields a bound on the excess population risk.

To evaluate the effectiveness of our approach, we conduct experiments on four public datasets. We find that BOPL outperforms boosted reward regression on three of these, thus illustrating that it can be advantageous to perform policy optimization, since it directly optimizes an estimate of the quantity we care about, the expected reward. We also find that boosted ensemble policies are competitive with neural network-based policies---while requiring less tuning, shorter training time and fewer computing resources---thereby demonstrating that boosted off-policy learning is a compelling alternative to deep off-policy learning.

\section{RELATED WORK}
\label{sec:related_work}

Of the extensive literature on learning from logged bandit feedback \citep{strehl:nips10,dudik:icml11,bottou:jmlr13,swaminathan:jmlr15,swaminathan:nips15,wu:icml18,ma:aistats19,kallus:jasa19,london:icml19,chen:nips19,faury:aaai20,jeunen:kdd20}, one particularly influential prior work, BanditNet \citep{joachims:iclr18}, can be viewed as the deep learning analog of our boosting approach. The authors argue that complex policy classes can overfit the logged propensities. To combat this phenomenon, they propose optimizing a self-normalized estimator, which they show is equivalent to optimizing an unnormalized estimator with translated rewards. For this reason, we explicitly allow rewards to be negative so as to accommodate negative translations---which we find to be critically important in our experiments.

The connection between early work on boosting \citep{kearns:stoc89,schapire:mlj90,freund:ic95,freund:jcss97} and functional gradient descent was formalized by \citet{friedman:techreport99} and \citet{mason:nips99}. This spawned a number of methods known collectively as \emph{gradient boosting}. One notable descendent of this line of work is XGBoost \citep{chen:kdd16}, which has become a popular choice for practitioners due to its speed and efficacy with minimal parameter tuning. XGBoost derives from a second-order Taylor approximation of an arbitrary loss function, and is optimized for regression tree base learners. While our boosting objective could in theory be approximated by XGBoost, our smoothness-based derivation gives us a closed-form expression for the ensemble weight of any base learner (not just regression trees) and an empirical risk bound, thus motivating a custom algorithm.

Learning from logged bandit feedback can be cast as a form of cost-sensitive classification, wherein labels have different costs depending on the context. Boosting algorithms for cost-sensitive classification have been proposed \citep{fan:icml99,abe:kdd04,appel:corr16}. However, the primary difference between cost-sensitive learning and learning from bandit feedback is that, in the former, it is assumed that all costs are known \emph{a priori}, whereas in the latter, only the cost of the selected action is observed.

The prior works that are most related to our own involve boosting with bandit feedback. Boosting has been applied to online multiclass classification with bandit feedback \citep{chen:icml14,zhang:aistats19}, online convex optimization with bandit feedback \citep{brukhim:alt21} and online reinforcement learning \citep{abel:corr16,brukhim:corr21}. The primary difference between these works and our own is that they consider an online setting in which the learning algorithm can interact directly with the environment and observe the corresponding outcome. In contrast, we consider an offline setting in which the learner cannot interact with the environment, and must therefore rely on logged interactions to estimate how the learned policy will perform online. The latter learning problem is inherently counterfactual, which is why off-policy corrections (e.g., importance weighting) are needed.

\section{PRELIMINARIES}
\label{sec:prelims}

Let $\Contexts$ denote a set of \emph{contexts}, and let $\Actions$ denote a set of \emph{actions} (sometimes referred to as \emph{arms}).\footnote{Though it is often the case that $\Actions$ depends on the context, to simplify notation, we assume that $\Actions$ is static.} We are interested in learning a \emph{policy}, $\Policy$, which defines a (stochastic) mapping from contexts to actions. Given a context, $\Context \in \Contexts$, we use $\Policy(\Context)$ to denote the policy's conditional probability distribution on $\Actions$, and $\Policy(\Action \| \Context)$ to denote the conditional probability of a given $\Action \in \Actions$. If $\Policy$ is deterministic, its distribution is a delta function.

A policy interacts with the environment in the following process. At each interaction, the environment generates a context, $\Context \in \Contexts$, according to a stationary distribution, $\ContextDist$. The policy responds by selecting (sampling) an action, $\Action \by \Policy(\Context)$, and consequently receives a stochastic \emph{reward}, $\Reward(\Context, \Action) \in \Reals$, which measures how good the selected action is for the given context. Importantly, we only observe reward for actions the policy selects. This partial supervision is referred to as \emph{bandit feedback}. We view the reward as a random function, $\Reward : \Contexts \times \Actions \to \Reals$, drawn from a stationary distribution, $\RewardDist$. Note that we do not assume that the rewards are nonnegative, to allow for the possibility that they could be \emph{baselined} (i.e., translated). This has been shown to be a crucial tool in preventing the so-called \emph{propensity overfitting} problem with complex model classes \citep{joachims:iclr18}. 

We want to find a policy that maximizes expected reward. In the sequel, it will be more convenient to think in terms of minimization, so we instead seek a policy with minimum expected \emph{negative} reward, which we call \emph{risk}:
\begin{equation}
    \Risk(\Policy) \defeq
        \Ep_{\Context \by \ContextDist}
        \Ep_{\Action \by \Policy(\Context)} 
        \Ep_{\Reward \by \RewardDist}
        [ -\Reward(\Context, \Action) ] .
\label{eq:risk}
\end{equation}
Note that an optimal policy is one that always selects an action with maximum mean reward, $\argmax_{\Action \in \Actions} \Ep_{\Reward \by \RewardDist}[ \Reward(\Context, \Action) ]$. However, such a policy may not be in the class of policies under consideration; and even if one is, it may be impossible to find given finite training data. We discuss this further in \cref{sec:off_policy_learning}, which motivates our off-policy learning objective.

\subsection{Softmax Ensemble Policies}
\label{sec:softmax_ensemble_policies}

Let $\Predictors \subseteq \{ \Contexts \times \Actions \to \Reals \}$ denote a class of \emph{predictors}, each of which maps context-action pairs to real-valued scores. For instance, $\Predictors$ could be a class of decision trees. It will at times be more convenient to think of a predictor, $\Predictor \in \Predictors$, as a vector-valued function that, given $\Context$, outputs the scores for all $\Action \in \Actions$ simultaneously, denoted by $\Predictor(\Context) \in \Reals^{\card{\Actions}}$.

For a collection of $\Rounds$ predictors, $\Predictor_1, \dots, \Predictor_{\Rounds} \in \Predictors$, with associated weights, $\EnsembleWeight_1, \dots, \EnsembleWeight_{\Rounds} \in \Reals$, let
\begin{equation}
    \Ensemble(\Context, \Action)
        \defeq \sum_{\Round=1}^{\Rounds} \EnsembleWeight_{\Round} \Predictor_{\Round}(\Context, \Action)
\label{eq:ensemble}
\end{equation}
denote the \emph{ensemble prediction}. We use
$
\Ensembles \defeq \big\{
    (\Context, \Action) \mapsto \Ensemble(\Context, \Action) :
    \forall \Round, \,
    \Predictor_{\Round} \in \Predictors, \,
    \EnsembleWeight_{\Round} \in \Reals
\big\}
$
to denote the class of size-$\Rounds$ ensembles, and the shorthand $\Ensemble \in \Ensembles$ to denote a member of this class.

Since we are not interested in predictions, but rather in a policy that selects actions, we use the following \emph{softmax} transformation. It can be applied to individual predictors or ensemble predictors. In particular, for a given ensemble, $\Ensemble \in \Ensembles$, and $\InvTemp \geq 0$, the corresponding \emph{softmax ensemble policy} is
\begin{equation}
    \EnsemblePolicy(\Action \| \Context)
    \defeq \Policy(\Action \| \Context ; \Ensemble, \InvTemp)
    \defeq \frac{
        \exp\left( \InvTemp \Ensemble(\Context, \Action) \right)
    }{
        \sum_{\Action' \in \Actions} \exp\left( \InvTemp \Ensemble(\Context, \Action') \right)
    } .
\label{eq:softmax_ensemble_policy}
\end{equation}
We use
$
\EnsemblePolicies \defeq \big\{
    (\Context, \Action) \mapsto \Policy(\Action \| \Context ; \Ensemble, \InvTemp) : \Ensemble \in \Ensembles
\big\}
$
to denote the class of softmax ensemble policies (for a given $\InvTemp \geq 0$), and the shorthand $\EnsemblePolicy \in \EnsemblePolicies$ to denote a member of this class.

The hyperparameter $\InvTemp$ (known as the \emph{inverse temperature}) is only used for notational convenience, as $\InvTemp \to \infty$ transforms the softmax into an \emph{argmax}. Importantly, for any $\InvTemp < \infty$, the softmax is differentiable, Lipschitz and (as we will later show) smooth, which are useful properties for optimization. For any softmax ensemble with finite $\InvTemp \neq 1$, an equivalent policy with $\InvTemp = 1$ can be obtained by rescaling the ensemble weights. Thus, during learning, we set $\InvTemp = 1$.

\subsection{Off-Policy Learning}
\label{sec:off_policy_learning}

Assume that we have collected data using an existing policy (not necessarily a softmax), $\LogPolicy$, which we call the \emph{logging policy}. Our only requirement for $\LogPolicy$ is that it has \emph{full support}; meaning, $\LogPolicy(\Action \| \Context) > 0$ for all $\Context \in \Contexts$ and $\Action \in \Actions$. For $i = 1, \dots, \DataSize$ interactions, we log: the context, $\Context_i \by \ContextDist$; the selected action, $\Action_i \by \LogPolicy(\Context_i)$; the \emph{propensity} of the selected action, $\Propensity_i \defeq \LogPolicy(\Action_i \| \Context_i)$; and the observed reward, $\Reward_i \defeq \Reward(\Context_i, \Action_i)$, where $\Reward \by \RewardDist$. The resulting dataset, $\Data \defeq (\Context_i, \Action_i, \Propensity_i, \Reward_i)_{i=1}^{\DataSize}$, can be used to train and evaluate new policies.

Recall that our goal is to minimize the risk (\cref{eq:risk}), and that the policy that achieves this is easy to derive if we have a perfect model of the mean rewards. In particular, if our hypothesis space is the class of softmax ensemble policies, $\EnsemblePolicies$, then an optimal policy is given by $\Ensemble(\Context, \Action) = \Ep_{\Reward \by \RewardDist}[ \Reward(\Context, \Action) ]$ and $\InvTemp \to \infty$. This motivates a strategy wherein we train an ensemble to predict the mean reward---an approach sometimes called \emph{reward regression}, or \emph{Q-learning} in the reinforcement learning literature. If we assume that the observed rewards are Gaussian perturbations of the mean reward, then a (nonlinear) least-squares regression makes sense:
\begin{equation}
    \min_{\Ensemble \in \Ensembles} \, \frac{1}{\DataSize} \sum_{i=1}^{\DataSize} (\Ensemble(\Context_i, \Action_i) - \Reward_i)^2 .
\label{eq:reward_regression_objective}
\end{equation}
If $\Ensembles$ contains the mean reward function, then this optimization will find an optimal ensemble (hence, optimal policy) in the limit of $\DataSize \to \infty$.

Unfortunately, in realistic scenarios, neither of these assumptions hold, and the relationship between regression error and expected reward can quickly become vacuous \citep{beygelzimer:kdd09}. Though $\Ensembles$ may be a very expressive function class, it is still unlikely that it contains the mean reward function. Further, while logged data may be abundant in industrial settings (such as search or recommendation engines), it may never be large enough; or it may be impractical to train on an extremely large dataset.

The key drawback of reward regression is that it does not directly optimize the quantity we actually care about: the risk (or, expected reward) of the learned policy when deployed. To see why, note that a reward predictor can improve its fit of the logged data and yet still not produce a policy with lower risk (i.e., higher expected reward). This can happen if the fit improves on actions that the new policy is unlikely to select. More perverse is the case where a trade-off is made that improves the fit on actions the new policy is unlikely to select while deteriorating the fit on actions \emph{it is} likely to select. In this case, reducing the squared error on the logged data can result in a worse policy.

We therefore consider a different approach based on \emph{policy optimization}. Unlike reward regression, policy optimization fits a policy to directly minimize risk, $\min_{\Policy \in \Policies} \Risk(\Policy)$. The true risk is unobservable, but given data, $\Data$, we can optimize an empirical estimate, $\EmpRisk(\Policy, \Data)$. This yields an \emph{empirical risk minimization} (ERM) strategy, $\min_{\Policy \in \Policies} \EmpRisk(\Policy, \Data)$.

The choice of risk estimator is crucial; if the estimator is biased, we may end up minimizing the wrong objective. Fortunately, we can obtain an unbiased estimate using the \emph{inverse propensity scoring} (IPS) estimator,
\begin{equation}
    \EmpRisk(\Policy, \Data) \defeq \frac{1}{\DataSize} \sum_{i=1}^{\DataSize} -\Reward_i \, \frac{\Policy(\Action_i \| \Context_i)}{\Propensity_i} .
\label{eq:ips_estimator}
\end{equation}
When $\LogPolicy$ has full support, it is straightforward to verify that $\Ep_{\Data}[\EmpRisk(\Policy, \Data)] = \Risk(\Policy)$. However, if the propensities can be very small, then the IPS estimator has large variance. For this reason, it is common practice to either design a logging policy whose propensities are lower-bounded, or to truncate the \emph{importance weights}, $\Policy(\Action_i \| \Context_i) / \Propensity_i$ \citep{ionides:jcgs08}. Alternatively, we could use an estimator that better balances the bias-variance trade-off, such as \emph{self-normalizing IPS} \citep{swaminathan:nips15} or \emph{doubly robust} \citep{dudik:icml11}. For simplicity, we will proceed with the regular IPS estimator, noting that more sophisticated estimators are complementary to our proposed approach.

In the following section, it will be convenient to view the empirical risk as the average \emph{loss} of an ensemble. For $i = 1, \dots, \DataSize$, let
$
    \Loss_i(\Ensemble)
    \defeq -\frac{\Reward_i}{\Propensity_i} \Policy(\Action_i \| \Context_i ; \Ensemble, \InvTemp)
$
denote the loss of $\Ensemble$ with respect to the $i\nth$ example. Using this notation, our learning objective is
\begin{equation}
    \min_{\EnsemblePolicy \in \EnsemblePolicies} \EmpRisk(\EnsemblePolicy, \Data) =
    \min_{\Ensemble \in \Ensembles} \frac{1}{\DataSize} \sum_{i=1}^{\DataSize} \Loss_i(\Ensemble) .
\label{eq:policy_learning_objective}
\end{equation}

\section{BOOSTED OFF-POLICY LEARNING}
\label{sec:boosted_policy_learning}

We will approach the optimization problem in \cref{eq:policy_learning_objective} using a greedy strategy known as \emph{boosting}, which can be viewed as coordinate descent or functional gradient descent \citep{schapire:book12}. Boosting proceeds in \emph{rounds}, wherein at each round, $\Round = 1, \dots, \Rounds$, the goal is to select a new predictor, $\Predictor_{\Round} \in \Predictors$, and corresponding weight, $\EnsembleWeight_{\Round} \in \Reals$, to add to our current ensemble, $\Ensemble[\Round-1]$, such that the empirical risk is minimized. Formally, the optimization problem at round $\Round$ is
\begin{align}
    \min_{
        \substack{\Predictor_{\Round} \in \Predictors, \\ \EnsembleWeight_{\Round} \in \Reals
    }}
        \EmpRisk(\EnsemblePolicy[\Round], \Data)
    = 
    \min_{
        \substack{\Predictor_{\Round} \in \Predictors, \\ \EnsembleWeight_{\Round} \in \Reals
    }} \,
        \frac{1}{\DataSize} \sum_{i=1}^{\DataSize} \Loss_i(\Ensemble[\Round-1] + \EnsembleWeight_{\Round} \Predictor_{\Round}) .
\label{eq:boosting_objective}
\end{align}

In the following, we will ignore the $\InvTemp$ hyperparameter of the softmax, essentially assuming $\InvTemp = 1$.

\subsection{Deriving the BOPL Algorithm}
\label{sec:algorithm}

\begin{algorithm*}
    \caption{Boosted Off-Policy Learning (BOPL)}
    \label{alg:boosted_policy_learning}
    \begin{algorithmic}[1]
    \Require{predictor class, $\Predictors$; base learner; rounds, $\Rounds \geq 1$; scale, $\NormConst > 0$}
    \State $\Ensemble[0] \gets 0$
    \For{$\Round = 1, \dots, \Rounds$}
        \State $\Predictor_{\Round} \gets \argmax_{\Predictor \in \Predictors} \ab{
            \frac{1}{\DataSize} \sum_{i=1}^{\DataSize} \frac{\Reward_i}{\Propensity_i} \EnsemblePolicy[\Round-1](\Action_i \| \Context_i) (\ActionOneHot_i - \EnsemblePolicy[\Round-1](\Context_i))\T \Predictor(\Context_i)
        }$
        \Comment{base learner}
        \Statex \quad\quad\quad\quad s.t. ~ $\frac{1}{\DataSize} \sum_{i=1}^{\DataSize} \frac{\ab{\Reward_i}}{\Propensity_i} \norm{\Predictor(\Context_i)}^2 = \NormConst$
        \State $\EnsembleWeight_{\Round} \gets \frac{2}{\DataSize \NormConst} \sum_{i=1}^{\DataSize} \frac{\Reward_i}{\Propensity_i} \EnsemblePolicy[\Round-1](\Action_i \| \Context_i) (\ActionOneHot_i - \EnsemblePolicy[\Round-1](\Context_i))\T \Predictor_{\Round}(\Context_i)$
        \State $\Ensemble[\Round] \gets \Ensemble[\Round-1] + \EnsembleWeight_{\Round} \Predictor_{\Round}$
    \EndFor
    \end{algorithmic}
\end{algorithm*}

Due to space constraints, we provide only a sketch of the derivation here, deferring the full derivation to \cref{sec:derivation_boosted_policy_learning}. The crux of the derivation (and subsequent analysis) is the \emph{smoothness} of the loss function, $\Loss_i$. Informally, a differentiable function is $\Smoothness$-smooth if its gradient is Lipschitz. In \cref{sec:smoothness}, we prove a new upper bound on the smoothness coefficient of the softmax function---which is, to our knowledge, the tightest such bound, and may be of independent interest. Using this result, we prove that $\Loss_i$ is $\frac{\ab{\Reward_i}}{2 \Propensity_i}$-smooth.

Having established smoothness, we construct a recursive upper bound for each $\Loss_i(\Ensemble[\Round])$ that isolates the influence of $\EnsembleWeight_{\Round}$ and $\Predictor_{\Round}$:
\begin{align}
    \Loss_i(\Ensemble[\Round])
    ~\leq~
    &   ~\Loss_i(\Ensemble[\Round-1]) \\
    &   - \frac{\Reward_i}{\Propensity_i} \EnsemblePolicy[\Round-1](\Action_i \| \Context_i) (\ActionOneHot_i - \EnsemblePolicy[\Round-1](\Context_i))\T (\EnsembleWeight_{\Round} \Predictor_{\Round}(\Context_i)) \\
    &    + \frac{\ab{\Reward_i}}{4 \Propensity_i} \norm{\EnsembleWeight_{\Round} \Predictor_{\Round}(\Context_i)}^2 ,
\end{align}
where $\ActionOneHot_i$ denotes a \emph{one-hot} encoding of the logged action, $\Action_i$, as a vector. Averaging over $i = 1, \dots, \DataSize$, we obtain a recursive upper bound on the empirical risk, $\EmpRisk(\EnsemblePolicy[\Round], \Data)$.

From there, we obtain a closed-form expression for the ensemble weight that minimizes the upper bound, for any given predictor:
\begin{equation}
    \EnsembleWeight_{\Round}^\star \defeq
    \frac{
        \frac{2}{\DataSize} \sum_{i=1}^{\DataSize} \frac{\Reward_i}{\Propensity_i} \EnsemblePolicy[\Round-1](\Action_i \| \Context_i) (\ActionOneHot_i - \EnsemblePolicy[\Round-1](\Context_i))\T \Predictor_{\Round}(\Context_i)
    }{
        \frac{1}{\DataSize} \sum_{i=1}^{\DataSize} \frac{\ab{\Reward_i}}{\Propensity_i} \norm{\Predictor_{\Round}(\Context_i)}^2
    } .
\end{equation}
Observe that the numerator is proportional to a weighted average,
\begin{equation}
    \sum_{i=1}^{\DataSize} \ImportanceWeight_i \sgn(\Reward_i) \Big(
        \Predictor_{\Round}(\Context_i, \Action_i)
        -
        \Ep_{\Action \by \EnsemblePolicy[\Round-1](\Context_i)}[\Predictor_{\Round}(\Context_i, \Action)]
    \Big) ,
\label{eq:weak_learning_condition}
\end{equation}
where $\ImportanceWeight_i \propto \frac{2 \ab{\Reward_i}}{\DataSize \Propensity_i} \EnsemblePolicy[\Round-1](\Action_i \| \Context_i) \geq 0$. If the weighted-average difference between $\Predictor_{\Round}(\Context_i, \Action_i)$ and the mean of $\Predictor_{\Round}(\Context_i, \Action)$ over $\Action \by \EnsemblePolicy[\Round-1](\Context_i)$ is nonzero, then $\EnsembleWeight_{\Round}^\star \neq 0$ and $\Predictor_{\Round}$ will contribute to the ensemble. This is analogous to AdaBoost's \emph{weak learning} condition: boosting can proceed as long as the base learner performs better than random guessing (i.e., accuracy greater than $1/2$) under a weighted empirical distribution. In fact, when $\Predictor_{\Round}$ is a $\TwoClass$-valued classifier, our weak learning condition coincides with AdaBoost's (see \cref{sec:boosting_via_binary_classification}). If the weak learning condition is not satisfied, then the new predictor adds no value to the ensemble, and boosting should terminate.

Plugging $\EnsembleWeight_{\Round}^\star$ into the recursive upper bound, we find that the predictor that minimizes the upper bound is
\begin{equation}
    \autofiteqn{
    \Predictor_{\Round}^\star
    \in \argmax_{\Predictor \in \Predictors} \,
    \frac{
        \left( \frac{1}{\DataSize} \sum_{i=1}^{\DataSize} \frac{\Reward_i}{\Propensity_i} \EnsemblePolicy[\Round-1](\Action_i \| \Context_i) (\ActionOneHot_i - \EnsemblePolicy[\Round-1](\Context_i))\T \Predictor(\Context_i) \right)^2
    }{
         \frac{1}{\DataSize} \sum_{i=1}^{\DataSize} \frac{\ab{\Reward_i}}{\Propensity_i} \norm{\Predictor(\Context_i)}^2
    } .
    }
\end{equation}
Critically, $\Predictor_{\Round}^\star$ is \emph{invariant to scaling}; meaning, for any $c > 0$, $c \Predictor_{\Round}^\star$ is still optimal. Thus, we can fix its scale by constraining $\frac{1}{\DataSize} \sum_{i=1}^{\DataSize} \frac{\ab{\Reward_i}}{\Propensity_i} \norm{\Predictor(\Context_i)}^2 = \NormConst$, for any $\NormConst > 0$, and simply maximize the numerator. This results in a constrained optimization problem that we refer to as the \emph{base learning objective}, which is solved by a \emph{base learner} for the given class of predictors, $\Predictors$. We give two examples of base learners in \cref{sec:base_learners}.

The resulting boosting algorithm, dubbed \emph{BOPL}, is given in \cref{alg:boosted_policy_learning}. Line 3 is the base learning objective. In practice, an early stopping condition can be inserted at the end of each iteration. If the gradient or $\EnsembleWeight_{\Round}$ is zero, then no further progress can be made. Alternatively, one could stop when a validation metric indicates overfitting. One could also apply regularization to the base learner or the ensemble weights; we discuss this in \cref{sec:regularization}.

\subsection{Analysis of BOPL}
\label{sec:analysis}

Using the analysis from our smoothness-based derivation (\cref{sec:derivation_boosted_policy_learning}), we can upper-bound BOPL's excess empirical risk; that is, the learned policy's suboptimality relative to an empirically optimal policy. The proof is deferred to \cref{sec:proof_emp_risk_bound}. 

\begin{theorem}
\label{th:emp_risk_bound}
Given a dataset, $\Data$, let $\OptEmpRisk \defeq \inf_{\Policy} \EmpRisk(\Policy, \Data)$ denote the minimum empirical risk, and $\ExcessEmpRisk_0 \defeq \EmpRisk(\EnsemblePolicy[0], \Data) - \OptEmpRisk$ the excess empirical risk of the initial (uniformly random) ensemble policy, $\EnsemblePolicy[0]$. If \cref{alg:boosted_policy_learning} is run for $\Rounds > 0$ rounds with $\NormConst > 0$, producing policy $\EnsemblePolicy$, then
\begin{equation}
    \EmpRisk(\EnsemblePolicy, \Data) - \OptEmpRisk
    \leq \ExcessEmpRisk_0 \exp\left(
        -\frac{\NormConst}{4 \ExcessEmpRisk_0} \sum_{\Round=1}^{\Rounds} \EnsembleWeight_{\Round}^2
    \right) .
\label{eq:emp_risk_bound}
\end{equation}
\end{theorem}

The bound decreases as long as each $\EnsembleWeight_{\Round} \neq 0$, which happens when \cref{eq:weak_learning_condition} is nonzero---our analog of the weak learning condition. If this condition is met at every round, then the excess empirical risk eventually converges to a stationary point. Further, if $\ab{\EnsembleWeight_{\Round}} \geq \Advantage > 0$ for all $\Round = 1, \dots, \Rounds$, then the bound is exponentially decreasing in $\Rounds$.\footnote{As shown in \cref{sec:boosting_via_binary_classification}, when $\Predictor_{\Round}$ is $\TwoClass$-valued, this condition is equivalent to always having error rate less than some $\Advantage < 1/2$ under a weighted empirical distribution.}

Though $\NormConst$ appears to be a free parameter, recall that $\NormConst$ implicitly defines the scale of the base learner, and that any tuning of $\NormConst$ will be automatically compensated for in the ensemble weights. Thus, one should instead think of $\NormConst$ as a property of the base learner.

Note that the infimum, $\inf_{\Policy}$, used to define $\OptEmpRisk$ is not constrained to any particular class of policies. Indeed, the empirically optimal policy may not be a member of $\EnsemblePolicies$, the class of ensemble policies. Thus, there may implicitly be some approximation gap, $\inf_{\EnsemblePolicy \in \EnsemblePolicies} \EmpRisk(\EnsemblePolicy, \Data) - \OptEmpRisk$, governed by the expressive power of $\EnsemblePolicies$ and the hardness of the dataset, which is difficult to quantify. That being said, note that $\OptEmpRisk$ and $\ExcessEmpRisk_0$ are straightforward to compute for a given dataset.

In \cref{sec:excess_risk_bound}, we show how \cref{th:emp_risk_bound} can be used to upper-bound BOPL's excess risk with respect to an optimal policy (i.e., risk minimizer). By decomposing the excess risk into several error terms---estimation error, generalization error and excess empirical risk---we can leverage existing tools for concentration and uniform convergence to upper-bound the first two terms. We provide an example bound that is monotonically decreasing in $\DataSize$ and $\Rounds$ when the Rademacher complexity of $\Predictors$ is $\LittleO(1)$ and the weak learning condition holds.

\subsection{Surrogate Objective and BOPL-S}
\label{sec:surrogate_objective}

\begin{algorithm*}
    \caption{Boosted Off-Policy Learning with Surrogate Objective (BOPL-S)}
    \label{alg:surrogate_boosted_policy_learning}
    \begin{algorithmic}[1]
    \Require{predictor class, $\Predictors$; base learner; rounds, $\Rounds \geq 1$; scale, $\NormConst > 0$}
    \State $\Ensemble[0] \gets 0$
    \For{$\Round = 1, \dots, \Rounds$}
        \State $\Predictor_{\Round} \gets \argmax_{\Predictor \in \Predictors} \ab{
            \frac{1}{\DataSize} \sum_{i=1}^{\DataSize} \frac{\Reward_i \GradSwitch_{i,t}}{\Propensity_i} (\ActionOneHot_i - \EnsemblePolicy[\Round-1](\Context_i))\T \Predictor(\Context_i)
        }$
        \Comment{base learner}
        \Statex \quad\quad\quad\quad s.t. ~ $\frac{1}{\DataSize} \sum_{i=1}^{\DataSize} \frac{\ab{\Reward_i} \Smoothness_i}{\Propensity_i} \norm{\Predictor(\Context_i)}^2 = \NormConst$
        \Statex \quad\quad\quad\quad with ~
            $
            \GradSwitch_{i,t} \gets
                \begin{cases}
                \EnsemblePolicy[\Round-1](\Action_i \| \Context_i) & \text{if } \Reward_i < 0 \\
                1 & \text{if } \Reward_i \geq 0
                \end{cases}
            $
            \quad and \quad
            $
            \Smoothness_i \gets
                \begin{cases}
                \frac{1}{2} & \text{if } \Reward_i < 0 \\
                1 & \text{if } \Reward_i \geq 0
                \end{cases}
            $
        \State $\EnsembleWeight_{\Round} \gets \frac{1}{\DataSize \NormConst} \sum_{i=1}^{\DataSize} \frac{\Reward_i \GradSwitch_{i,t}}{\Propensity_i} (\ActionOneHot_i - \EnsemblePolicy[\Round-1](\Context_i))\T \Predictor_{\Round}(\Context_i)$
        \State $\Ensemble[\Round] \gets \Ensemble[\Round-1] + \EnsembleWeight_{\Round} \Predictor_{\Round}$
    \EndFor
    \end{algorithmic}
\end{algorithm*}

From the viewpoint of functional gradient descent \citep{friedman:techreport99,mason:nips99}, the objective function in \cref{eq:policy_learning_objective} may be difficult to optimize for two reasons. First, the softmax function is not convex in $\Ensemble$, and there are exponentially many local optima \citep{chen:nips19}, so convergence to a global optimum is not guaranteed. Second, the gradient of the softmax with respect to $\Ensemble$ vanishes quickly, meaning the optimization can get ``stuck" early on.

To circumvent these issues (in certain cases), we propose to boost a surrogate objective. For the time being, we assume that the $i\nth$ reward is nonnegative, $\Reward_i \geq 0$. Then, using the identity $-c (\ln z + 1) \geq -c z$, for all $c \in \Reals_+$ and $z \in \Reals$, we define a \emph{surrogate loss function},
\begin{equation}
    \LogLoss_i(\Ensemble[\Round])
    \defeq -\frac{\Reward_i}{\Propensity_i} ( \ln \Policy(\Action_i \| \Context_i ; \Ensemble[\Round]) + 1 )
    \geq \Loss_i(\Ensemble[\Round]) ,
\label{eq:surrogate_log_loss}
\end{equation}
variants of which have been used in the literature on policy optimization \citep{leroux:corr16,ma:aistats19,london:icml19,jeunen:kdd20}. When $\Reward_i \geq 0$, we have that $\LogLoss_i(\Ensemble[\Round])$ is convex in $\Ensemble[\Round]$ and upper-bounds $\Loss_i(\Ensemble[\Round])$; moreover, $\LogLoss_i(\Ensemble[\Round])$ has the same minimum as $\Loss_i(\Ensemble[\Round])$, and its gradient is zero only at the minimum. Like the original loss function, the surrogate loss function is smooth---albeit with a different coefficient, which we bound in \cref{sec:smoothness}. We can therefore establish a recursive upper bound on the surrogate loss; and hence, on the empirical risk.

However, we started out by assuming that the reward is nonnegative, and this is not always true---especially if we wish to support reward translation. When the reward is negative, $\LogLoss_i$ is no longer an upper bound for $\Loss_i$, nor is it convex. To handle this case, we can resort to the original loss function. For the $i\nth$ training example, we use $\Loss_i$ when the reward is negative, and $\LogLoss_i$ when the reward is nonnegative. The resulting surrogate objective is
\begin{equation}
    \autofiteqn[.9]{
    \SurrogateEmpRisk(\EnsemblePolicy[\Round], \Data)
    \defeq \frac{1}{\DataSize} \sum_{i=1}^{\DataSize}
        \1\{ \Reward_i < 0 \} \Loss_i(\Ensemble[\Round])
        + \1\{ \Reward_i \geq 0 \} \LogLoss_i(\Ensemble[\Round]) ,
    }
\label{eq:surrogate_objective}
\end{equation}
which is an upper bound for $\EmpRisk(\EnsemblePolicy[\Round], \Data)$ because $\Loss_i(\Ensemble[\Round]) \leq \LogLoss_i(\Ensemble[\Round])$ when $\Reward_i \geq 0$. In the presence of negative rewards, this function is still non-convex, and still suffers from vanishing gradients. However, our hope is that the examples with nonnegative rewards contribute enough gradient to keep the optimization moving in the right direction.

Applying our smoothness analysis to the surrogate objective, we derive \cref{alg:surrogate_boosted_policy_learning}, which we call \emph{BOPL-S} (for \emph{surrogate}). See \cref{sec:derivation_surrogate_boosted_policy_learning} for the full derivation.

\subsection{Base Learners}
\label{sec:base_learners}

\cref{alg:boosted_policy_learning,alg:surrogate_boosted_policy_learning} depend on a base learner to solve the optimization in line 3. This algorithm will depend, to some extent, on the class of predictors, $\Predictors$. We now sketch base learning reductions (full details given in \cref{sec:base_learning_reductions}) for two classes, real-valued functions (e.g., regression trees) and binary classifiers (e.g., decision stumps), both of which can be implemented by a variety of readily available, off-the-shelf tools. In both cases, we assume that $\Predictors$ is \emph{symmetric}; meaning, for every $\Predictor \in \Predictors$, we also have $-\Predictor \in \Predictors$.

\paragraph{Regression.}
Assume that $\Predictors$ is a symmetric class of real-valued functions, $\Predictors \subseteq \{ \Contexts \times \Actions \to \Reals \}$. Since we assume that $\Predictors$ is symmetric, we can omit the absolute value from the base learning objective. We then convert the constrained optimization problem to the following unconstrained one via Lagrangian relaxation:
\begin{align}
    \argmax_{\Predictor \in \Predictors} \min_{\Lagrange \in \Reals} ~
        &\frac{1}{\DataSize} \sum_{i=1}^{\DataSize} \frac{\Reward_i}{\Propensity_i} \EnsemblePolicy[\Round-1](\Action_i \| \Context_i) (\ActionOneHot_i - \EnsemblePolicy[\Round-1](\Context_i))\T \Predictor(\Context_i) \\
        &- \Lagrange \left( \frac{1}{\DataSize} \sum_{i=1}^{\DataSize} \frac{\ab{\Reward_i}}{\Propensity_i} \norm{\Predictor(\Context_i)}^2 - \NormConst \right) .
\end{align}
For every $\NormConst$, there exists a $\Lagrange$ that is optimal, and vice versa. Since $\NormConst$ is arbitrary, we can choose any $\Lagrange$ for the optimization. We therefore take $\Lagrange = 1/2$ and, after dropping terms that are irrelevant to the optimization and rearranging the expression, we obtain a weighted least-squares regression,
\begin{align}
    &\argmin_{\Predictor \in \Predictors} ~
        \sum_{i=1}^{\DataSize} \sum_{\Action \in \Actions} \ImportanceWeight_i \left( \PseudoLabel_{i,\Action} - \Predictor(\Context_i, \Action) \right)^2 ,
\end{align}
with nonnegative weights, $\ImportanceWeight_i \defeq \frac{\ab{\Reward_i}}{\Propensity_i}$, and pseudo-labels, 
\begin{equation}
    \PseudoLabel_{i, \Action} \defeq \sgn(\Reward_i) \EnsemblePolicy[\Round-1](\Action_i \| \Context_i) ( \1\{ \Action = \Action_i \} - \EnsemblePolicy[\Round-1](\Action \| \Context_i) ) .
\end{equation}
Note that this optimization does not explicitly constrain the scale of the predictor to a given $\NormConst$. Thus, for \cref{th:emp_risk_bound} to hold, the predictor can optionally be rescaled. The full details of the reduction and the resulting algorithm are given in \cref{sec:boosting_via_regression}.

\paragraph{Binary Classification.}
Assume that $\Predictors$ a symmetric class of binary classifiers, $\Predictors \subseteq \{ \Contexts \times \Actions \to \TwoClass \}$. For every example, $i \in \{1, \dots, \DataSize\}$, and action, $\Action \in \Actions$, we define a nonnegative weight,
\begin{equation}
    \ImportanceWeight_{i, \Action} \defeq \ab{
        \frac{\Reward_i}{\Propensity_i} \EnsemblePolicy[\Round-1](\Action_i \| \Context_i) ( \1\{ \Action = \Action_i \} - \EnsemblePolicy[\Round-1](\Action \| \Context_i) )
    } ,
\end{equation}
and a $\TwoClass$-valued pseudo-label,
\begin{equation}
    \PseudoLabel_{i, \Action}
    \defeq \sgn(\Reward_i) ( 2 \, \1\{ \Action = \Action_i \} - 1 ) .
\end{equation}
Then, using the identities
\begin{equation}
    \ImportanceWeight_{i, \Action} \PseudoLabel_{i, \Action}
    = \frac{\Reward_i}{\Propensity_i} \EnsemblePolicy[\Round-1](\Action_i \| \Context_i) ( \1\{ \Action = \Action_i \} - \EnsemblePolicy[\Round-1](\Action \| \Context_i) )
\end{equation}
and
\begin{equation}
    \1\{ \PseudoLabel_{i, \Action} \neq \Predictor(\Context_i, \Action) \}
    =
    \frac{1}{2} ( 1 - \PseudoLabel_{i, \Action} \Predictor(\Context_i, \Action) ) ,
\end{equation}
we derive an equivalence between the base learning objective and minimizing the weighted classification error:
\begin{align}
    &~ \argmax_{\Predictor \in \Predictors} \ab{ \frac{1}{\DataSize} \sum_{i=1}^{\DataSize} \frac{\Reward_i}{\Propensity_i} \EnsemblePolicy[\Round-1](\Action_i \| \Context_i) (\ActionOneHot_i - \EnsemblePolicy[\Round-1](\Context_i))\T \Predictor(\Context_i) } \\
    = &~ \argmin_{\Predictor \in \Predictors} \sum_{i=1}^{\DataSize} \sum_{\Action \in \Actions} \ImportanceWeight_{i, \Action} \1\{ \PseudoLabel_{i, \Action} \neq \Predictor(\Context_i, \Action) \} .
\end{align}
This optimization problem can be (approximately) solved by any learning algorithm for weighted binary classification. Since the predictions are $\TwoClass$-valued, every $\Predictor \in \Predictors$ satisfies $\norm{\Predictor(\Context)}^2 = \card{\Actions}$. Thus, the base learning scale constraint is automatically satisfied by $\NormConst = \frac{\card{\Actions}}{\DataSize} \sum_{i=1}^{\DataSize} \frac{\ab{\Reward_i}}{\Propensity_i}$, which is nonnegative as long as there is at least one nonzero reward in the dataset. The full details of the reduction and the resulting algorithm are given in \cref{sec:boosting_via_binary_classification}.

\section{EXPERIMENTS}
\label{sec:experiments}

The following experiments are designed to answer several questions: first, whether boosted ensemble policies are competitive with other complex policy classes, such as deep neural networks; second, whether BOPL's off-policy learning objective (\cref{eq:policy_learning_objective}) is more effective at training ensemble policies than the reward regression objective (\cref{eq:reward_regression_objective}); and finally, whether BOPL-S's surrogate learning objective (\cref{eq:surrogate_objective}) is easier to optimize.

\begin{table*}[ht]
    \caption{Average reward on the test data, averaged over 10 trials, with $95\%$ confidence intervals. Best scores for each dataset are bolded. Multiple bold cells occur when confidence intervals overlap with those of the best score.}
    \centering
    \begin{tabular}{lcccc}
    \toprule
     & \bf{Covertype} & \bf{Fashion-MNIST} & \bf{Scene} &  \bf{TMC2007-500} \\
    \midrule
    \bf{Logging} & $0.4907 ~ \pm 0.0024$ & $0.4708 ~ \pm 0.0009$ & $0.4317 ~ \pm 0.0042$ & $0.3187 ~ \pm 0.0029$ \\
    \midrule
    \bf{BRR-gb} & $0.9033 ~ \pm 0.0015$ & $0.8652 ~ \pm 0.0012$ & $0.6328 ~ \pm 0.0319$ & $0.7288 ~ \pm 0.0046$ \\
    \bf{BRR-xgb} & $0.9465 ~ \pm 0.0004$ & $0.8739 ~ \pm 0.0011$ & $0.6854 ~ \pm 0.0189$ & $\bf{0.7742 ~ \pm 0.0034}$ \\
    \bf{DRR} & $0.8878 ~ \pm 0.0011$ & $\bf{0.8947 ~ \pm 0.0019}$ & $0.7636 ~ \pm 0.0098$ & $\bf{0.7787 ~ \pm 0.0038}$ \\
    \bf{BanditNet} & $0.8565 ~ \pm 0.0019$ & $\bf{0.8921 ~ \pm 0.0025}$ & $\bf{0.7794 ~ \pm 0.0124}$ & $0.7603 ~ \pm 0.0076$ \\
    \midrule
    \bf{BOPL-regr} & $0.9504 ~ \pm 0.0005$ & $0.8893 ~ \pm 0.0016$ & $\bf{0.7778 ~ \pm 0.0105}$ & $0.7361 ~ \pm 0.0047$ \\
    \bf{BOPL-class} & $0.9262 ~ \pm 0.0010$ & $0.8775 ~ \pm 0.0012$ & $0.7651 ~ \pm 0.0112$ & $0.7131 ~ \pm 0.0039$ \\
    \bf{BOPL-S-regr} & $\bf{0.9531 ~ \pm 0.0007}$ & $0.8876 ~ \pm 0.0022$ & $\bf{0.7703 ~ \pm 0.0132}$ & $0.7399 ~ \pm 0.0030$ \\
    \bf{BOPL-S-class} & $0.9191 ~ \pm 0.0010$ & $0.8759 ~ \pm 0.0014$ & $\bf{0.7877 ~ \pm 0.0091}$ & $0.7339 ~ \pm 0.0051$ \\
    \bottomrule
    \end{tabular}
    \label{tab:final_results}
\end{table*}

\begin{table*}[ht]
    \caption{%
    Direct method (DM) reward estimation versus actual test reward.
    }
    \centering
    \begin{tabular}{lcccc}
    \toprule
     & \multicolumn{2}{c}{\bf{Scene}} & \multicolumn{2}{c}{\bf{TMC2007-500}} \\
     & \bf{DM Reward} & \bf{True Reward} & \bf{DM Reward} & \bf{True Reward} \\
    \midrule
    \bf{BRR-xgb} & $\bf{0.5126 ~ \pm 0.0103}$ & $0.6854 ~ \pm 0.0189$ & $\bf{0.7609 ~ \pm 0.0027}$ & $\bf{0.7742 ~ \pm 0.0034}$ \\
    \bf{BOPL-regr} & $0.4972 ~ \pm 0.0099$ & $\bf{0.7778 ~ \pm 0.0105}$ & $0.7023 ~ \pm 0.0034$ & $0.7361 ~ \pm 0.0047$ \\
    \bottomrule
    \end{tabular}
    \label{tab:dm_est_vs_true_reward}
\end{table*}

\subsection{Data and Methodology}
\label{sec:data_and_methodology}

We use the standard supervised-to-bandit conversion \citep{beygelzimer:kdd09} to simulate bandit feedback, since it allows us to compute accurate evaluation metrics from the ground-truth data. We report results on four public datasets, which have been used in related work \citep{swaminathan:jmlr15,london:icml19}: Covertype \citep{covertype}, Fashion-MNIST \citep{fashionmnist}, Scene \citep{scene} and TMC2007-500 \citep{tmc2007500}. Details about the datasets, partitioning, preprocessing and task reward structures are given in \cref{sec:dataset_details}.

For each method, we tune its associated hyperparameters via random search over a grid, whose limits were determined using the validation reward. Each hyperparameter configuration is used to train a policy on $10$ simulated bandit datasets; then, each policy is evaluated on the validation data. The configuration with the highest average validation reward is selected, and its associated policies are evaluated on the testing data. Final rewards are averaged over the $10$ trials (i.e., $10$ trained policies).

With the exception of the logging policy, we evaluate all policies as argmax policies (i.e., softmax with $\InvTemp \to \infty$), which deterministically selects the highest scoring (or, most likely) action. This eliminates the effect of exploration, and puts stochastic and deterministic policies on the same footing.

\subsection{Algorithms Tested}
\label{sec:algos_tested}

Both \cref{alg:boosted_policy_learning} (BOPL) and \cref{alg:surrogate_boosted_policy_learning} (BOPL-S) can be instantiated with two base learners, regression (suffix \emph{-regr}) or binary classification (\emph{-class}), for a total of four combinations. We use decision tree base learners throughout. Since boosted ensembles are complex model classes, we combat propensity overfitting using reward translation \citep{joachims:iclr18}.

To validate BOPL's learning objective, which directly optimizes an estimate of the ensemble policy's expected reward---rather than reward prediction error---we compare BOPL to several baselines based on reward regression (\cref{eq:reward_regression_objective}). The first two baselines, prefixed \emph{BRR} (boosted reward regression), use boosted regression trees, so as to control for the model class and isolate the effect of the learning objective. \emph{BRR-gb} is our own implementation of a ``vanilla" gradient boosting algorithm \citep{friedman:techreport99} for least-squares regression. \emph{BRR-xgb} is the highly optimized and state-of-the-art \emph{XGBoost} \citep{chen:kdd16} with the squared error loss. Both algorithms have a tunable learning rate (a.k.a. \emph{shrinkage} parameter), which effectively scales the ensemble weights\footnote{We tried a learning rate with BOPL and BOPL-S as well, but found that it was not helpful.}. The third baseline, \emph{DRR} (deep reward regression), trains a neural network regressor. Thus, DRR differs from BOPL in both the model class and learning objective. Note that none of the reward regression baselines require reward translation.

Finally, to focus on the question of boosting versus deep learning for off-policy learning, we compare to \emph{BanditNet} \citep{joachims:iclr18}. Like BOPL, BanditNet trains a softmax policy by optimizing the IPS estimator (\cref{eq:ips_estimator}), with reward translation to combat propensity overfitting. However, BanditNet's underlying model (i.e., action scoring function) is a neural network. Thus, this comparison controls for the learning objective and isolates the effect of the model class. 

Further details about the algorithm implementations and hyperparameters can be found in \cref{sec:algo_impl_details}.

\subsection{Results}
\label{sec:results}

\cref{tab:final_results} summarizes the results of our experiments. Overall, on two of the four datasets, at least one of the proposed BOPL variants performs best, and the other BOPL variants are typically competitive as well. In particular, BOPL-S-regr is always either the best or statistically tied with the best variant. The following investigates our stated research questions in greater detail.

\paragraph{Is boosting competitive with deep learning?}
While BOPL significantly outperforms all baselines on Covertype, the deep learning methods, DRR and BanditNet, outperform other methods on Fashion-MNIST. These results are not that surprising, as tree-based models are known to perform well on tabular data \citep{grinsztajn:corr22} and deep learning is known to excel at computer vision problems. We also find that BanditNet is statistically tied with BOPL on Scene, and BRR-xgb is statistically tied with DRR on TMC2007-500. It is worth noting, however, that the deep learning methods have more hyperparameters to tune, take much longer to train, and consume more resources---requiring GPUs to run in a reasonable amount of time, while the boosting methods only use CPUs. Indeed, boosted tree ensembles live up to to their reputation of performing robustly without complex tuning or massive computing resources. From this perspective, boosting not only performs competitively (if not better), it is often a more \emph{practical} choice.

\paragraph{How does BOPL compare to reward regression?}
Overall, DRR is the strongest reward regression baseline, and it outperforms BOPL on two datasets. However, BOPL outperforms BRR (which controls for the model class) on three out of four datasets. To understand why, consider the fact that BOPL directly maximizes the IPS estimate of the policy's expected reward, while BRR can be viewed as \emph{indirectly} optimizing a \emph{different} reward estimator: the so-called \emph{direct method} (DM). Instead of importance weighting, DM uses a reward regressor to predict the reward for each action selected by the target policy. Clearly, the policy that maximizes the DM estimate is one that always picks the best action according to the reward regressor. Therefore, whether BOPL or BRR performs better can now be understood in terms of the bias-variance trade-offs of the IPS and DM estimators.

To illustrate this point, \cref{tab:dm_est_vs_true_reward} analyzes the bias of the DM estimator (using BRR-xgb) on the two datasets where BOPL and BRR perform best and worst, respectively. On Scene, where BRR performs poorly, the DM estimate of the test reward is highly biased; it predicts $0.5126$ for the BRR-xgb policy when the true reward is $0.6854$, and it is even less accurate when estimating the reward of the BOPL-regr policy. This means that picking a policy based on the DM estimator is bound to be unreliable, which explains BRR's poor performance on Scene. On the other hand, the DM estimator is substantially less biased on TMC2007-500, which explains BRR's relatively good performance on that dataset. Unfortunately, it is difficult to determine when, and to what degree, the DM estimator will be biased, so it is difficult to say ahead of time whether reward regression will be successful. In contrast, the IPS estimator (used in BOPL) has reliable tools to control the bias-variance trade-off \citep{dudik:icml11,swaminathan:nips15}.

\paragraph{Is BOPL-S's surrogate objective easier to optimize?}
Recall that BOPL's learning objective is non-convex, and its gradients tend to vanish quickly. This motivated BOPL-S, which boosts a surrogate objective that partially addresses BOPL's issues. For any example with nonnegative rewards, BOPL-S's objective is convex and has a non-vanishing gradient. Unfortunately, examples with negative rewards are subject to the same problems as BOPL, but our hope is that correcting the nonnegative-reward examples will help. To validate this claim, we compare BOPL and BOPL-S on the Covertype data, since this is the dataset where BOPL-S has a statistically significant advantage. Since the rewards for this dataset are nonnegative, negative reward can only come from reward translation. We therefore evaluate two variants of BOPL-S: one with reward translation, and one without called \emph{BOPL-CS} since its surrogate objective is convex. \cref{fig:grad_vs_reward} plots the norm of the gradient, as well as the average reward on the training and testing data, as they evolve over a single run of each algorithm.\footnote{We use a single run to illustrate a typical optimization trajectory in which BOPL-S outperforms BOPL. Given the inter-run variability of the optimization, averaging over multiple runs might wash out any non-monotonicity in the trajectories of individual runs, which would hide the phenomena we are trying to present.} We find that BOPL’s gradient is almost monotonically decreasing in magnitude, meaning the optimization is largely decelerating, and its rewards have a bumpy upward trajectory compared to the surrogate objectives. In contrast, BOPL-CS’s gradient shoots up in early rounds, meaning the optimization is making fast progress. Indeed, we see BOPL-CS's rewards rise faster than BOPL's. BOPL-S’s gradient looks like a somewhat ``softened” version of BOPL-CS’s. This is likely because, with reward translation, some examples still use BOPL's objective. That being said, BOPL-S ultimately achieves the highest reward. Therefore, while convexity is helpful, reward translation may be more so.

\begin{figure}[ht]
    \centering
    \includegraphics[width=\columnwidth]{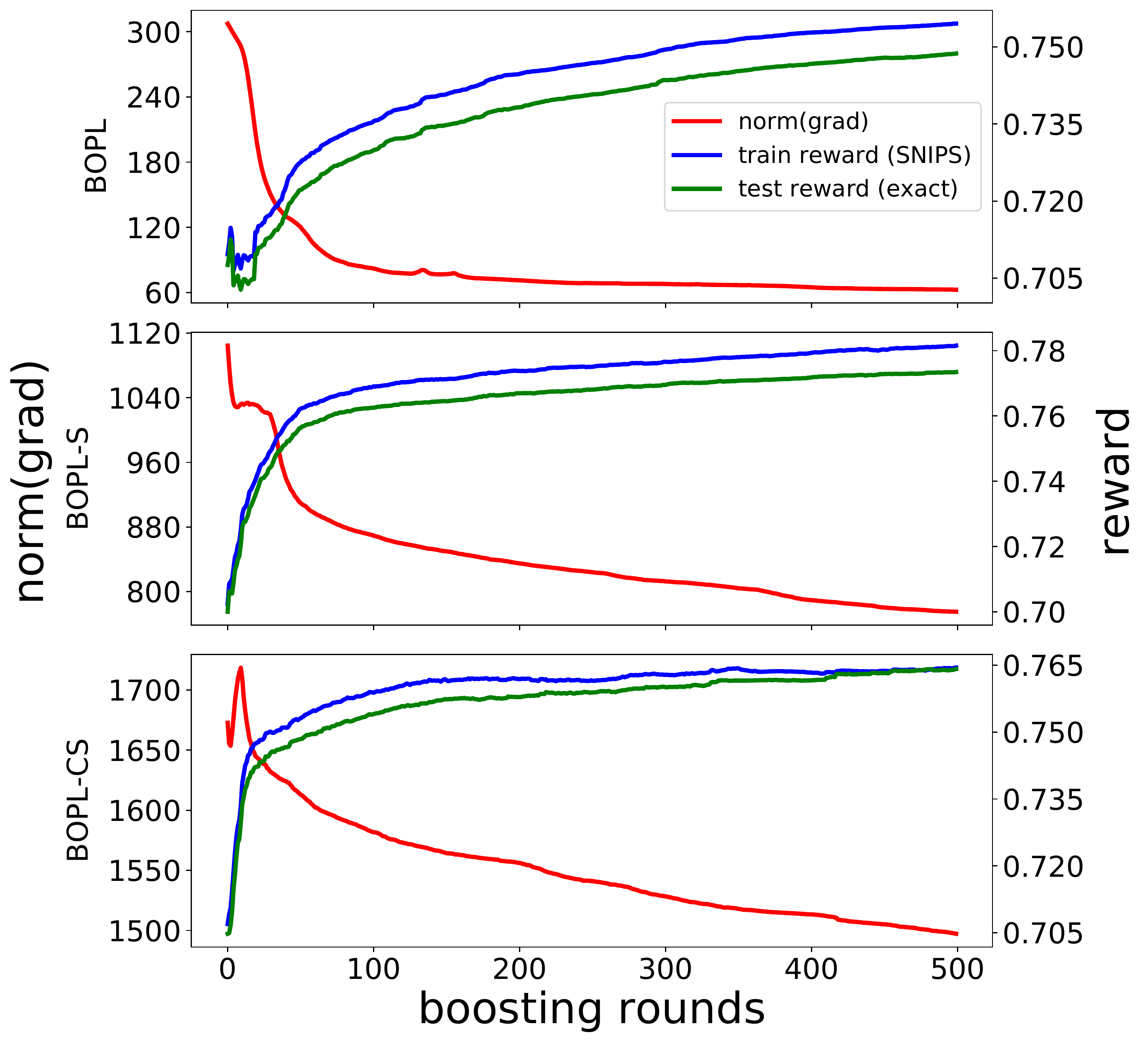}
    \caption{
        Plots of the gradient norm and train/test reward over a single run of BOPL, BOPL-S and BOPL-CS on the Covertype data. We use the self-normalized IPS estimator \citep{swaminathan:nips15} to estimate reward on the training data. Note that the hyperparameters used in these plots are not the optimized ones used in \cref{tab:final_results}.
    }
    \label{fig:grad_vs_reward}
\end{figure}

\section{CONCLUSIONS AND FUTURE WORK}
\label{sec:conclusions}

We proposed the first boosting algorithm for off-policy learning from logged bandit feedback. Both our theory and our empirical results highlight the importance of boosting the right off-policy objective, and suggest that our boosting algorithm is a robust alternative to deep off-policy learning. This novel ability to directly optimize a policy's expected reward via boosting opens a new space of research questions and further improvements. In particular, one could continue to refine the algorithm with more sophisticated surrogate objectives (e.g., \citep{leroux:corr16,chen:nips19}), optimize estimators beyond vanilla IPS (e.g., \citep{dudik:icml11,swaminathan:nips15}) so as to reduce variance, or examine off-policy regularizers that do not easily decompose over the ensemble (e.g., \citep{swaminathan:jmlr15}).

\subsubsection*{Acknowledgements}

We thank Emily Butler, Bob Kaucic, Gert Lanckriet and Mani Sethuraman from Amazon Music, as well as Leo Dirac, Avi Geiger and Paulina Varshavskaya from Groundlight.ai, for supporting this work.


\bibliographystyle{plainnat}
\bibliography{biblio}

\begin{thebibliography}{54}
\providecommand{\natexlab}[1]{#1}
\providecommand{\url}[1]{\texttt{#1}}
\expandafter\ifx\csname urlstyle\endcsname\relax
  \providecommand{\doi}[1]{doi: #1}\else
  \providecommand{\doi}{doi: \begingroup \urlstyle{rm}\Url}\fi

\bibitem[Abe et~al.(2004)Abe, Zadrozny, and Langford]{abe:kdd04}
N.~Abe, B.~Zadrozny, and J.~Langford.
\newblock An iterative method for multi-class cost-sensitive learning.
\newblock In \emph{Knowledge Discovery and Data Mining}, 2004.

\bibitem[Abel et~al.(2016)Abel, Agarwal, Diaz, Krishnamurthy, and
  Schapire]{abel:corr16}
D.~Abel, A.~Agarwal, F.~Diaz, A.~Krishnamurthy, and R.~Schapire.
\newblock Exploratory gradient boosting for reinforcement learning in complex
  domains.
\newblock \emph{CoRR}, abs/1603.04119, 2016.

\bibitem[Agarwal et~al.(2020)Agarwal, Kakade, Lee, and Mahajan]{agarwal:colt20}
A.~Agarwal, S.~Kakade, J.~Lee, and G.~Mahajan.
\newblock Optimality and approximation with policy gradient methods in {M}arkov
  decision processes.
\newblock In \emph{Conference on Learning Theory}, 2020.

\bibitem[Appel et~al.(2016)Appel, Burgos-Artizzu, and Perona]{appel:corr16}
R.~Appel, X.~Burgos-Artizzu, and P.~Perona.
\newblock Improved multi-class cost-sensitive boosting via estimation of the
  minimum-risk class.
\newblock \emph{CoRR}, abs/1607.03547, 2016.

\bibitem[Bartlett \& Mendelson(2003)Bartlett and Mendelson]{bartlett:jmlr03}
P.~Bartlett and S.~Mendelson.
\newblock {R}ademacher and {G}aussian complexities: Risk bounds and structural
  results.
\newblock \emph{Journal of Machine Learning Research}, 3, 2003.

\bibitem[Ben-David et~al.(2003)Ben-David, Eiron, and Long]{ben-david:jcss03}
S.~Ben-David, N.~Eiron, and P.~Long.
\newblock On the difficulty of approximately maximizing agreements.
\newblock \emph{Journal of Computer and System Sciences}, 66\penalty0 (3),
  2003.

\bibitem[Beygelzimer \& Langford(2009)Beygelzimer and
  Langford]{beygelzimer:kdd09}
A.~Beygelzimer and J.~Langford.
\newblock The offset tree for learning with partial labels.
\newblock In \emph{Knowledge Discovery and Data Mining}, 2009.

\bibitem[Blackard \& Dean(1999)Blackard and Dean]{covertype}
J.~Blackard and D.~Dean.
\newblock Comparative accuracies of artificial neural networks and discriminant
  analysis in predicting forest cover types from cartographic variables.
\newblock \emph{Computers and electronics in agriculture}, 24\penalty0 (3),
  1999.

\bibitem[Bottou \& Bousquet(2007)Bottou and Bousquet]{bottou:nips07}
L.~Bottou and O.~Bousquet.
\newblock The tradeoffs of large scale learning.
\newblock In \emph{Neural Information Processing Systems}, 2007.

\bibitem[Bottou et~al.(2013)Bottou, Peters, Qui{\~n}onero-Candela, Charles,
  Chickering, Portugaly, Ray, Simard, and Snelson]{bottou:jmlr13}
L.~Bottou, J.~Peters, J.~Qui{\~n}onero-Candela, D.~Charles, D.~Chickering,
  E.~Portugaly, D.~Ray, P.~Simard, and E.~Snelson.
\newblock Counterfactual reasoning and learning systems: The example of
  computational advertising.
\newblock \emph{Journal of Machine Learning Research}, 14, 2013.

\bibitem[Boutell et~al.(2004)Boutell, Luo, Shen, and Brown]{scene}
M.~Boutell, J.~Luo, X.~Shen, and C.~Brown.
\newblock Learning multi-label scene classification.
\newblock \emph{Pattern recognition}, 37\penalty0 (9), 2004.

\bibitem[Brukhim \& Hazan(2021)Brukhim and Hazan]{brukhim:alt21}
N.~Brukhim and E.~Hazan.
\newblock Online boosting with bandit feedback.
\newblock In \emph{Conference on Algorithmic Learning Theory}, 2021.

\bibitem[Brukhim et~al.(2021)Brukhim, Hazan, and Singh]{brukhim:corr21}
N.~Brukhim, E.~Hazan, and K.~Singh.
\newblock A boosting approach to reinforcement learning.
\newblock \emph{CoRR}, abs/2108.09767, 2021.

\bibitem[Chen et~al.(2019)Chen, Gummadi, Harris, and Schuurmans]{chen:nips19}
M.~Chen, R.~Gummadi, C.~Harris, and D.~Schuurmans.
\newblock Surrogate objectives for batch policy optimization in one-step
  decision making.
\newblock In \emph{Neural Information Processing Systems}, 2019.

\bibitem[Chen et~al.(2014)Chen, Lin, and Lu]{chen:icml14}
S.-T. Chen, H.-T. Lin, and C.-J. Lu.
\newblock Boosting with online binary learners for the multiclass bandit
  problem.
\newblock In \emph{International Conference on Machine Learning}, 2014.

\bibitem[Chen \& Guestrin(2016)Chen and Guestrin]{chen:kdd16}
T.~Chen and C.~Guestrin.
\newblock {XGBoost}: A scalable tree boosting system.
\newblock In \emph{Knowledge Discovery and Data Mining}, 2016.

\bibitem[Chen et~al.(2015)Chen, M., Li, Lin, Wang, Wang, Xiao, Xu, Zhang, and
  Zhang]{mxnet}
T.~Chen, M., Y.~Li, M.~Lin, N.~Wang, M.~Wang, T.~Xiao, B.~Xu, C.~Zhang, and
  Z.~Zhang.
\newblock {MXNet}: A flexible and efficient machine learning library for
  heterogeneous distributed systems.
\newblock \emph{CoRR}, abs/1512.01274, 2015.

\bibitem[Duchi \& Singer(2009)Duchi and Singer]{duchi:icml09}
J.~Duchi and Y.~Singer.
\newblock Boosting with structural sparsity.
\newblock In \emph{International Conference on Machine Learning}, 2009.

\bibitem[Duchi et~al.(2011)Duchi, Hazan, and Singer]{duchi:jmlr11}
J.~Duchi, E.~Hazan, and Y.~Singer.
\newblock Adaptive subgradient methods for online learning and stochastic
  optimization.
\newblock \emph{Journal of Machine Learning Research}, 12:\penalty0 2121--2159,
  2011.

\bibitem[Dudik et~al.(2011)Dudik, Langford, and Lihong]{dudik:icml11}
M.~Dudik, J.~Langford, and L.~Lihong.
\newblock Doubly robust policy evaluation and learning.
\newblock In \emph{International Conference on Machine Learning}, 2011.

\bibitem[Fan et~al.(1999)Fan, Stolfo, Zhang, and Chan]{fan:icml99}
W.~Fan, S.~Stolfo, J.~Zhang, and P.~Chan.
\newblock {AdaCost}: Misclassification cost-sensitive boosting.
\newblock In \emph{International Conference on Machine Learning}, 1999.

\bibitem[Faury et~al.(2020)Faury, Tanielian, Vasile, Smirnova, and
  Dohmatob]{faury:aaai20}
L.~Faury, U.~Tanielian, F.~Vasile, E.~Smirnova, and E.~Dohmatob.
\newblock Distributionally robust counterfactual risk minimization.
\newblock In \emph{AAAI}, 2020.

\bibitem[Feurer et~al.(2019)Feurer, van Rijn, Kadra, Gijsbers, Mallik, Ravi,
  M{\"u}ller, Vanschoren, and Hutter]{feurer:corr19}
M.~Feurer, J.~van Rijn, A.~Kadra, P.~Gijsbers, N.~Mallik, S.~Ravi,
  A.~M{\"u}ller, J.~Vanschoren, and F.~Hutter.
\newblock {OpenML}-{P}ython: an extensible {P}ython {API} for {OpenML}.
\newblock \emph{CoRR}, abs/1911.02490, 2019.

\bibitem[Freund(1995)]{freund:ic95}
Y.~Freund.
\newblock Boosting a weak learning algorithm by majority.
\newblock \emph{Information and Computation}, 121, 1995.

\bibitem[Freund \& Schapire(1997)Freund and Schapire]{freund:jcss97}
Y.~Freund and R.~Schapire.
\newblock A decision-theoretic generalization of on-line learning and an
  application to boosting.
\newblock \emph{Journal of Computer and System Sciences}, 55\penalty0 (1),
  1997.

\bibitem[Friedman(1999)]{friedman:techreport99}
J.~Friedman.
\newblock Greedy function approximation: a gradient boosting machine.
\newblock Technical report, Stanford University, 1999.

\bibitem[Gao \& Pavel(2017)Gao and Pavel]{gao:arxiv17}
B.~Gao and L.~Pavel.
\newblock On the properties of the softmax function with application in game
  theory and reinforcement learning.
\newblock \emph{ArXiv}, abs/1704.00805, 2017.

\bibitem[Grinsztajn et~al.(2022)Grinsztajn, Oyallon, and
  Varoquaux]{grinsztajn:corr22}
L.~Grinsztajn, E.~Oyallon, and G.~Varoquaux.
\newblock Why do tree-based models still outperform deep learning on tabular
  data?
\newblock \emph{CoRR}, abs/2207.08815, 2022.

\bibitem[Hoeffding(1963)]{hoeffding:jasa63}
W.~Hoeffding.
\newblock Probability inequalities for sums of bounded random variables.
\newblock \emph{Journal of the American Statistical Association}, 58\penalty0
  (301), 1963.

\bibitem[Ioffe \& Szegedy(2015)Ioffe and Szegedy]{ioffe:icml15}
S.~Ioffe and C.~Szegedy.
\newblock Batch normalization: Accelerating deep network training by reducing
  internal covariate shift.
\newblock In \emph{International Conference on Machine Learning}, 2015.

\bibitem[Ionides(2008)]{ionides:jcgs08}
E.~Ionides.
\newblock Truncated importance sampling.
\newblock \emph{Journal of Computational and Graphical Statistics}, 17\penalty0
  (2):\penalty0 295--311, 2008.

\bibitem[Jeunen et~al.(2020)Jeunen, Rohde, Vasile, and Bompaire]{jeunen:kdd20}
O.~Jeunen, D.~Rohde, F.~Vasile, and M.~Bompaire.
\newblock Joint policy-value learning for recommendation.
\newblock In \emph{Knowledge Discovery and Data Mining}, 2020.

\bibitem[Joachims et~al.(2018)Joachims, Swaminathan, and
  de~Rijke]{joachims:iclr18}
T.~Joachims, A.~Swaminathan, and M.~de~Rijke.
\newblock Deep learning with logged bandit feedback.
\newblock In \emph{International Conference on Learning Representations}, 2018.

\bibitem[Kallus(2019)]{kallus:jasa19}
N.~Kallus.
\newblock More efficient policy learning via optimal retargeting.
\newblock \emph{Journal of the American Statistical Association}, 116, 2019.

\bibitem[Kearns \& Valiant(1989)Kearns and Valiant]{kearns:stoc89}
M.~Kearns and L.~Valiant.
\newblock Cryptographic limitations on learning {B}oolean formulae and finite
  automata.
\newblock In \emph{Symposium on Theory of Computing}, 1989.

\bibitem[{Le Roux}(2016)]{leroux:corr16}
N.~{Le Roux}.
\newblock Efficient iterative policy optimization.
\newblock \emph{CoRR}, abs/1612.08967, 2016.

\bibitem[London \& Sandler(2019)London and Sandler]{london:icml19}
B.~London and T.~Sandler.
\newblock Bayesian counterfactual risk minimization.
\newblock In \emph{International Conference on Machine Learning}, 2019.

\bibitem[Ma et~al.(2019)Ma, Wang, and Narayanaswamy]{ma:aistats19}
Y.~Ma, Y.-X. Wang, and B.~Narayanaswamy.
\newblock Imitation-regularized offline learning.
\newblock In \emph{Artificial Intelligence and Statistics}, 2019.

\bibitem[Mason et~al.(1999)Mason, Baxter, Bartlett, and Frean]{mason:nips99}
L.~Mason, J.~Baxter, P.~Bartlett, and M.~Frean.
\newblock Boosting algorithms as gradient descent.
\newblock In \emph{Neural Information Processing Systems}, 1999.

\bibitem[Maurer(2016)]{maurer:alt16}
A.~Maurer.
\newblock A vector-contraction inequality for {R}ademacher complexities.
\newblock In \emph{Conference on Algorithmic Learning Theory}, 2016.

\bibitem[Mei et~al.(2020)Mei, Xiao, Szepesvari, and Schuurmans]{mei:icml20}
J.~Mei, C.~Xiao, C.~Szepesvari, and D.~Schuurmans.
\newblock On the global convergence rates of softmax policy gradient methods.
\newblock In \emph{International Conference on Machine Learning}, 2020.

\bibitem[Mohri et~al.(2012)Mohri, Rostamizadeh, and Talwalkar]{mohri:book12}
M.~Mohri, A.~Rostamizadeh, and A.~Talwalkar.
\newblock \emph{Foundations of Machine Learning}.
\newblock Adaptive computation and machine learning. {MIT} Press, 2012.
\newblock ISBN 978-0-262-01825-8.

\bibitem[Pedregosa et~al.(2011)Pedregosa, Varoquaux, Gramfort, Michel, Thirion,
  Grisel, Blondel, Prettenhofer, Weiss, Dubourg, Vanderplas, Passos,
  Cournapeau, Brucher, Perrot, and Duchesnay]{scikit-learn}
F.~Pedregosa, G.~Varoquaux, A.~Gramfort, V.~Michel, B.~Thirion, O.~Grisel,
  M.~Blondel, P.~Prettenhofer, R.~Weiss, V.~Dubourg, J.~Vanderplas, A.~Passos,
  D.~Cournapeau, M.~Brucher, M.~Perrot, and E.~Duchesnay.
\newblock Scikit-learn: Machine learning in {P}ython.
\newblock \emph{Journal of Machine Learning Research}, 12, 2011.

\bibitem[Schapire(1990)]{schapire:mlj90}
R.~Schapire.
\newblock The strength of weak learnability.
\newblock \emph{Machine Learning}, 5, 1990.

\bibitem[Schapire \& Freund(2012)Schapire and Freund]{schapire:book12}
R.~Schapire and Y.~Freund.
\newblock \emph{Boosting: Foundations and Algorithms}.
\newblock Adaptive computation and machine learning. MIT Press, 2012.
\newblock ISBN 9780262017183.

\bibitem[Srivastava \& Zane-Ulman(2005)Srivastava and Zane-Ulman]{tmc2007500}
A.~Srivastava and B.~Zane-Ulman.
\newblock Discovering recurring anomalies in text reports regarding complex
  space systems.
\newblock In \emph{Aerospace Conference}, 2005.

\bibitem[Strehl et~al.(2010)Strehl, Langford, Li, and Kakade]{strehl:nips10}
A.~Strehl, J.~Langford, L.~Li, and S.~Kakade.
\newblock Learning from logged implicit exploration data.
\newblock In \emph{Neural Information Processing Systems}, 2010.

\bibitem[Swaminathan \& Joachims(2015{\natexlab{a}})Swaminathan and
  Joachims]{swaminathan:jmlr15}
A.~Swaminathan and T.~Joachims.
\newblock Batch learning from logged bandit feedback through counterfactual
  risk minimization.
\newblock \emph{Journal of Machine Learning Research}, 2015{\natexlab{a}}.

\bibitem[Swaminathan \& Joachims(2015{\natexlab{b}})Swaminathan and
  Joachims]{swaminathan:nips15}
A.~Swaminathan and T.~Joachims.
\newblock The self-normalized estimator for counterfactual learning.
\newblock In \emph{Neural Information Processing Systems}, 2015{\natexlab{b}}.

\bibitem[Vanschoren et~al.(2013)Vanschoren, van Rijn, Bischl, and
  Torgo]{vanschoren:kdde13}
J.~Vanschoren, J.~van Rijn, B.~Bischl, and L.~Torgo.
\newblock {OpenML}: networked science in machine learning.
\newblock \emph{SIGKDD Explorations}, 15\penalty0 (2), 2013.

\bibitem[Wu \& Wang(2018)Wu and Wang]{wu:icml18}
H.~Wu and M.~Wang.
\newblock Variance regularized counterfactual risk minimization via variational
  divergence minimization.
\newblock In \emph{International Conference on Machine Learning}, 2018.

\bibitem[Xiao et~al.(2017)Xiao, Rasul, and Vollgraf]{fashionmnist}
H.~Xiao, K.~Rasul, and R.~Vollgraf.
\newblock Fashion-{MNIST}: a novel image dataset for benchmarking machine
  learning algorithms.
\newblock \emph{CoRR}, abs/1708.07747, 2017.

\bibitem[Yang et~al.(2020)Yang, Fan, Wu, and Udell]{yang:kdd20}
C.~Yang, J.~Fan, Z.~Wu, and M.~Udell.
\newblock {AutoML} pipeline selection: Efficiently navigating the combinatorial
  space.
\newblock In \emph{Knowledge Discovery and Data Mining}, 2020.

\bibitem[Zhang et~al.(2019)Zhang, Jung, and Tewari]{zhang:aistats19}
D.~Zhang, Y.~Jung, and A.~Tewari.
\newblock Online multiclass boosting with bandit feedback.
\newblock In \emph{Artificial Intelligence and Statistics}, 2019.

\end{thebibliography}


\clearpage

\appendix
\onecolumn

\section{SMOOTHNESS}
\label{sec:smoothness}

Our algorithm derivations rely on a property of the loss function known as \emph{smoothness}.

\begin{definition}
\label{def:smoothness}
A differentiable function, $\SmoothFunc : \SmoothFuncDomain \to \Reals$, is $\Smoothness$-\emph{smooth} if, for all $\SmoothFuncInput, \SmoothFuncInput' \in \SmoothFuncDomain$,
\begin{equation}
    \norm{\grad \SmoothFunc(\SmoothFuncInput) - \grad \SmoothFunc(\SmoothFuncInput')}^2 \leq \frac{\Smoothness}{2} \norm[2]{\SmoothFuncInput - \SmoothFuncInput'}^2 .
\end{equation}
\end{definition}

In this appendix, we prove that two loss functions---one used in \cref{alg:boosted_policy_learning}; the other used in \cref{alg:surrogate_boosted_policy_learning}---are smooth. We begin with several technical lemmas, from which the smoothness proofs directly follow.

\begin{lemma}
\label{lem:spectral_norm}
For a symmetric, rank-$1$ matrix, $\mat{A} = \vec{x} \vec{x}\T$, with $\vec{x} \in \Reals^{n \times n}$, its spectral norm (i.e., largest eigenvalue) is $\norm[2]{\mat{A}} = \norm[2]{\vec{x}}^2$.
\end{lemma}
\begin{proof}
By definition, the spectral norm of $\mat{A}$ is
\begin{align}
    \norm[2]{\mat{A}}
    &= \sup_{\UnitVec : \norm[2]{\UnitVec} = 1} \norm[2]{\mat{A} \UnitVec} \\
    &= \sup_{\UnitVec : \norm[2]{\UnitVec} = 1} \norm[2]{\vec{x} \vec{x}\T \UnitVec} \\
    &= \norm[2]{\vec{x}} \sup_{\UnitVec : \norm[2]{\UnitVec} = 1} (\vec{x}\T \UnitVec) \\
    &= \norm[2]{\vec{x}} \norm[2]{\vec{x}} .
\end{align}
The last equality follows from an alternate definition of the Euclidean norm.
\end{proof}

\begin{lemma}
\label{lem:softmax_smooth}
The Hessian of the softmax function, $\Policy(\Action \| \Context ; \Predictor, \InvTemp) = \frac{\exp(\InvTemp \Predictor(\Context, \Action))}{\sum_{\Action' \in \Actions} \exp(\InvTemp \Predictor(\Context, \Action'))}$, has bounded spectral norm:
\begin{equation}
    \norm[2]{\grad^2 \Policy(\Action \| \Context ; \Predictor, \InvTemp)}
    \leq 2 \InvTemp^2 \Policy(\Action \| \Context ; \Predictor, \InvTemp) \big( 1 - \Policy(\Action \| \Context ; \Predictor, \InvTemp) \big) .
\label{eq:softmax_hessian_norm_bound}
\end{equation}
Therefore, $\Policy(\Action \| \Context ; \Predictor, \InvTemp)$ is $\frac{\InvTemp^2}{2}$-smooth.
\end{lemma}
\begin{proof}
For a twice-differentiable function (like the softmax), smoothness is equivalent to having a uniformly upper-bounded second derivative. Consequently, our first step will be to derive a formula for the Hessian, $\grad^2 \Policy(\Action \| \Context ; \Predictor, \InvTemp)$, with respect to $\Predictor(\Context)$ (which, for boosting, is an ensemble prediction, $\Ensemble[\Round](\Context)$). We will show that the Hessian decomposes into the sum of two symmetric, rank-$1$ matrices. We then upper-bound the norm of each matrix using \cref{lem:spectral_norm} and some additional reasoning.

To simplify notation, we will omit $\Context$, $\Predictor$ and $\InvTemp$, and simply write $\Policy \defeq \Policy(\Context ; \Predictor, \InvTemp) \in \Reals^{\card{\Actions}}$ to denote the vector of softmax probabilities, and $\Policy(\Action) \defeq \Policy(\Action \| \Context ; \Predictor, \InvTemp)$ to denote the conditional probability of $\Action$.

Using the log derivative trick and the product rule, we have that
\begin{align}
    \grad^2 \Policy(\Action)
    &= \grad\big( \grad \Policy(\Action) \big) \\
    &= \grad\big( \Policy(\Action) \grad \ln \Policy(\Action) \big) \\
    &= (\grad \Policy(\Action)) (\grad \ln \Policy(\Action))\T + \Policy(\Action) \grad^2 \ln \Policy(\Action) \\
    &= \Policy(\Action) \Big(
        \underbrace{ (\grad \ln \Policy(\Action)) (\grad \ln \Policy(\Action))\T }_{\mat{H}_1}
        + \underbrace{ \grad^2 \ln \Policy(\Action) }_{\mat{H}_2}
    \Big) .
\end{align}
By the subadditivity of the norm, 
\begin{equation}
    \norm[2]{\grad^2 \Policy(\Action)}
    = \norm[2]{\Policy(\Action) (\mat{H}_1 + \mat{H}_2)}
    \leq \Policy(\Action) \big( \norm[2]{\mat{H}_1} + \norm[2]{\mat{H}_2} \big) ,
\end{equation}
so we can upper-bound the norms of $\mat{H}_1$ and $\mat{H}_2$ separately.

Via \cref{lem:spectral_norm}, we have that
\begin{equation}
    \norm[2]{\mat{H}_1}
    = \norm[2]{(\grad \ln \Policy(\Action)) (\grad \ln \Policy(\Action))\T}
    = \norm[2]{\grad \ln \Policy(\Action)}^2 .
\end{equation}
The gradient of $\ln \Policy(\Action)$ is $\InvTemp (\ActionOneHot - \Policy)$, where $\ActionOneHot$ (without subscript) denotes the one-hot encoding of $\Action$. Therefore,
\begin{equation}
    \norm[2]{\grad \ln \Policy(\Action)}^2
    = \InvTemp^2 \norm[2]{\ActionOneHot - \Policy}^2
    = \InvTemp^2 \big( \ActionOneHot\T \ActionOneHot - 2 \ActionOneHot\T \Policy + \Policy\T \Policy \big)
    = \InvTemp^2 \big( 1 - 2 \Policy(\Action) + \Policy\T \Policy \big) .
\end{equation}

Turning now to $\mat{H}_2$, we note that $\grad^2 \ln \Policy(\Action_i)$ is the covariance of the distribution $\Policy$, scaled by $-\InvTemp^2$:
\begin{equation}
    \grad^2 \ln \Policy(\Action)
    = -\InvTemp^2 \Ep_{\Action' \by \Policy} \left[ (\ActionOneHot' - \Policy) (\ActionOneHot' - \Policy)\T \right] .
\label{eq:log_softmax_hessian}
\end{equation}
Thus,
\begin{align}
    \norm[2]{\mat{H}_2}
    &= \norm[2]{ \grad^2 \ln \Policy(\Action) } \\
    &= \InvTemp^2 \norm[2]{ \Ep_{\Action' \by \Policy} \left[ (\ActionOneHot' - \Policy) (\ActionOneHot' - \Policy)\T \right] } \\
    &\leq \InvTemp^2 \Ep_{\Action' \by \Policy} \left[ \norm[2]{ (\ActionOneHot' - \Policy) (\ActionOneHot' - \Policy)\T } \right] \\
    &= \InvTemp^2 \Ep_{\Action' \by \Policy} \left[ \norm[2]{ \ActionOneHot' - \Policy }^2 \right] .
\label{eq:norm_log_softmax_hessian_ub1}
\end{align}
The inequality is from Jensen's inequality; the final equality is from \cref{lem:spectral_norm}. Continuing,
\begin{equation}
    \Ep_{\Action' \by \Policy} \left[ \norm[2]{ \ActionOneHot' - \Policy }^2 \right]
    = \Ep_{\Action' \by \Policy} \left[ {\ActionOneHot'}\T \ActionOneHot' - 2 {\ActionOneHot'}\T \Policy + \Policy\T \Policy \right]
    = 1 - 2 \Ep_{\Action' \by \Policy}[\ActionOneHot']\T \Policy + \Policy\T \Policy
    = 1 - \Policy\T \Policy .
\label{eq:norm_log_softmax_hessian_ub2}
\end{equation}

Finally, combining the inequalities, we have
\begin{align}
    \norm[2]{\grad^2 \Policy(\Action)}
    &\leq \Policy(\Action) \big( \norm[2]{\mat{H}_1} + \norm[2]{\mat{H}_2} \big) \\
    &\leq \InvTemp^2 \Policy(\Action) \big( (1 - 2 \Policy(\Action) + \Policy\T \Policy) + (1 - \Policy\T \Policy) \big) \\
    &= 2 \InvTemp^2 \Policy(\Action) ( 1 - \Policy(\Action) ) ,
\end{align}
which proves \cref{eq:softmax_hessian_norm_bound}. To finish the proof, we note that
\begin{equation}
    2 \InvTemp^2 \Policy(\Action) ( 1 - \Policy(\Action) )
    \leq 2 \InvTemp^2 \times \frac{1}{4}
    = \frac{\InvTemp^2}{2} ,
\end{equation}
which follows from the fact that $\Policy(\Action) (1 - \Policy(\Action)) \leq \frac{1}{4}$ for $\Policy(\Action) \in [0, 1]$.
\end{proof}

\begin{remark}
To the best of our knowledge, \cref{lem:softmax_smooth} is the best (and possibly first documented) upper bound on the smoothness coefficient of the softmax. It complements work by \citet{gao:arxiv17}, who showed that the softmax is $\InvTemp$-Lipschitz. We are aware of only two other related smoothness results: \citet{agarwal:colt20} showed that the function $\Policy\T \vec{\Reward}$, for $\vec{\Reward} \in \Reals^{\card{\Actions}}$, is $5 \norm[\infty]{\vec{\Reward}}$-smooth; similarly, \citet{mei:icml20} showed that $\Policy\T \vec{\Reward}$, for $\vec{\Reward} \in [0, 1]^{\card{\Actions}}$, is $\frac{5}{2}$-smooth. \cref{eq:softmax_hessian_norm_bound} can be used to show that $\Policy\T \vec{\Reward}$, for $\vec{\Reward} \in \Reals^{\card{\Actions}}$, is actually $2 \norm[\infty]{\vec{\Reward}}$-smooth, thereby improving upon both prior results.
\end{remark}

\begin{lemma}
\label{lem:log_softmax_smooth}
The log-softmax function, $\ln \Policy(\Action \| \Context ; \Predictor, \InvTemp)$, is $\InvTemp^2 (1 - \card{\Actions}^{-1})$-smooth.
\end{lemma}
\begin{proof}
The proof builds on the proof of \cref{lem:softmax_smooth}. We will reuse our previous notation, where $\Policy(\Action) \defeq \Policy(\Action \| \Context ; \Predictor, \InvTemp)$ and $\Policy \defeq \Policy(\Context ; \Predictor, \InvTemp)$. Recall (from \cref{eq:log_softmax_hessian}) that the Hessian of the log-softmax is $\grad^2 \ln \Policy(\Action) = -\InvTemp^2 \Ep_{\Action' \by \Policy} \left[ (\ActionOneHot' - \Policy) (\ActionOneHot' - \Policy)\T \right]$, and (via \cref{eq:norm_log_softmax_hessian_ub1,eq:norm_log_softmax_hessian_ub2}) its norm has upper bound
\begin{equation}
    \norm[2]{\grad^2 \ln \Policy(\Action)} \leq \InvTemp^2 (1 - \Policy\T \Policy) .
\end{equation}
Since $\Policy$ is constrained to the simplex ($\norm[1]{\Policy} = 1$), it is straightforward to show that
\begin{equation}
    \inf_{\Policy : \norm[1]{\Policy} = 1} \norm[2]{\Policy}^2
    \geq \card{\Actions}^{-1} .
\end{equation}
Thus,
\begin{equation}
    \InvTemp^2 (1 - \Policy\T \Policy)
    \leq \InvTemp^2 (1 - \card{\Actions}^{-1}) ,
\end{equation}
which completes the proof.
\end{proof}

Having established that the softmax and log-softmax are smooth, we are now ready to prove our main smoothness results.

\begin{proposition}
\label{prop:loss_function_smooth}
The loss function, $\Loss_i(\Ensemble[\Round]) = -\frac{\Reward_i}{\Propensity_i} \Policy(\Action_i \| \Context_i ; \Ensemble[\Round])$, is $\frac{\ab{\Reward_i}}{2 \Propensity_i}$-smooth.
\end{proposition}
\begin{proof}
Since the softmax is $\frac{1}{2}$-smooth (for $\InvTemp = 1$), we have that
\begin{equation}
    \norm[2]{\grad^2 \Loss_i(\Ensemble[\Round])}
    = \frac{\ab{\Reward_i}}{\Propensity_i} \norm[2]{\grad^2 \Policy(\Action_i \| \Context_i ; \Ensemble[\Round])}
    \leq \frac{\ab{\Reward_i}}{2 \Propensity_i} .
\end{equation}
Thus completes the proof.
\end{proof}

\begin{proposition}
\label{prop:log_loss_function_smooth}
The surrogate loss function, $\LogLoss_i(\Ensemble[\Round]) = -\frac{\Reward_i}{\Propensity_i} ( \ln \Policy(\Action_i \| \Context_i ; \Ensemble[\Round]) + 1 )$, is $\frac{\ab{\Reward_i}}{\Propensity_i}$-smooth.
\end{proposition}
\begin{proof}
To prove \cref{prop:log_loss_function_smooth}, we will simplify \cref{lem:log_softmax_smooth} by noting that $\InvTemp^2 (1 - \card{\Actions}^{-1}) \leq \InvTemp^2$. Then, for $\InvTemp = 1$, the log-softmax is $1$-smooth. Therefore,
\begin{equation}
    \norm[2]{\grad^2 \LogLoss_i(\Ensemble[\Round])}
    = \frac{\ab{\Reward_i}}{\Propensity_i} \norm[2]{\grad^2 \ln \Policy(\Action_i \| \Context_i ; \Ensemble[\Round])}
    \leq \frac{\ab{\Reward_i}}{\Propensity_i} ,
\end{equation}
which completes the proof.
\end{proof}

\section{DERIVATIONS OF BOOSTING ALGORITHMS}
\label{sec:boosting_algo_dervations}

This appendix contains the full derivations of our boosting algorithms, which were deferred from the main paper.

\subsection{Derivation of \cref{alg:boosted_policy_learning} (BOPL)}
\label{sec:derivation_boosted_policy_learning}

In this section, we give an unabridged derivation of \cref{alg:boosted_policy_learning}. It all starts by applying our smoothness result (\cref{prop:loss_function_smooth}) to construct a recursive upper bound on $\Loss_i$ that isolates the influence of $\EnsembleWeight_{\Round}$ and $\Predictor_{\Round}$. To do so, we use the following technical lemma.

\begin{lemma}
\label{lem:smooth_upper_bound}
If $\SmoothFunc : \SmoothFuncDomain \to \Reals$ is $\Smoothness$-smooth, then for all $\SmoothFuncInput, \SmoothFuncInput' \in \SmoothFuncDomain$,
\begin{equation}
    \SmoothFunc(\SmoothFuncInput) \leq \SmoothFunc(\SmoothFuncInput') + \grad \SmoothFunc(\SmoothFuncInput')\T (\SmoothFuncInput - \SmoothFuncInput') + \frac{\Smoothness}{2} \norm[2]{\SmoothFuncInput - \SmoothFuncInput'}^2 .
\end{equation}
\end{lemma}

To apply \cref{lem:smooth_upper_bound}, first note that the gradient of $\Loss_i$ with respect to $\Ensemble[\Round](\Context_i)$ is 
\begin{equation}
    \grad \Loss_i(\Ensemble[\Round])
    = -\frac{\Reward_i}{\Propensity_i} \Policy(\Action_i \| \Context_i ; \Ensemble[\Round]) (\ActionOneHot_i - \Policy(\Context_i ; \Ensemble[\Round]))
    = -\frac{\Reward_i}{\Propensity_i} \EnsemblePolicy[\Round](\Action_i \| \Context_i) (\ActionOneHot_i - \EnsemblePolicy[\Round](\Context_i)) ,
\label{eq:loss_gradient_wrt_ensemble}
\end{equation}
where $\ActionOneHot_i$ denotes the one-hot encoding of $\Action_i$. Combining this with \cref{prop:loss_function_smooth,lem:smooth_upper_bound}, we have that
\begin{equation}
    \Loss_i(\Ensemble[\Round])
    \leq \Loss_i(\Ensemble[\Round-1])
        - \frac{\Reward_i}{\Propensity_i} \EnsemblePolicy[\Round-1](\Action_i \| \Context_i) (\ActionOneHot_i - \EnsemblePolicy[\Round-1](\Context_i))\T (\EnsembleWeight_{\Round} \Predictor_{\Round}(\Context_i))
        + \frac{\ab{\Reward_i}}{4 \Propensity_i} \norm{\EnsembleWeight_{\Round} \Predictor_{\Round}(\Context_i)}^2 .
\label{eq:loss_smooth_upper_bound}
\end{equation}
Averaging \cref{eq:loss_smooth_upper_bound} over $i = 1, \dots, \DataSize$, we obtain a recursive upper bound:
\begin{equation}
    \EmpRisk(\EnsemblePolicy[\Round], \Data)
    \leq \EmpRisk(\EnsemblePolicy[\Round-1], \Data)
        - \frac{1}{\DataSize} \sum_{i=1}^{\DataSize} \Bigg[
            \frac{\Reward_i}{\Propensity_i} \EnsemblePolicy[\Round-1](\Action_i \| \Context_i) (\ActionOneHot_i - \EnsemblePolicy[\Round-1](\Context_i))\T (\EnsembleWeight_{\Round} \Predictor_{\Round}(\Context_i))
            - \frac{\ab{\Reward_i}}{4 \Propensity_i} \norm{\EnsembleWeight_{\Round} \Predictor_{\Round}(\Context_i)}^2
        \Bigg] ,
\label{eq:emp_risk_quad_upper_bound}
\end{equation}

Observe that the bound is quadratic in the ensemble weight, $\EnsembleWeight_{\Round}$. Thus, we can obtain a closed-form expression for the ensemble weight that minimizes the upper bound, for any given predictor:
\begin{equation}
    \EnsembleWeight_{\Round}^\star \defeq
    \frac{
        \frac{2}{\DataSize} \sum_{i=1}^{\DataSize} \frac{\Reward_i}{\Propensity_i} \EnsemblePolicy[\Round-1](\Action_i \| \Context_i) (\ActionOneHot_i - \EnsemblePolicy[\Round-1](\Context_i))\T \Predictor_{\Round}(\Context_i)
    }{
        \frac{1}{\DataSize} \sum_{i=1}^{\DataSize} \frac{\ab{\Reward_i}}{\Propensity_i} \norm{\Predictor_{\Round}(\Context_i)}^2
    } .
\label{eq:optimal_ensemble_weight}
\end{equation}
Plugging $\EnsembleWeight_{\Round}^\star$ into \cref{eq:emp_risk_quad_upper_bound} yields another recursive upper bound:
\begin{equation}
    \EmpRisk(\EnsemblePolicy[\Round], \Data)
    \leq \EmpRisk(\EnsemblePolicy[\Round-1], \Data)
        -\frac{
            \left( \frac{1}{\DataSize} \sum_{i=1}^{\DataSize} \frac{\Reward_i}{\Propensity_i} \EnsemblePolicy[\Round-1](\Action_i \| \Context_i) (\ActionOneHot_i - \EnsemblePolicy[\Round-1](\Context_i))\T \Predictor_{\Round}(\Context_i) \right)^2
        }{
             \frac{1}{\DataSize} \sum_{i=1}^{\DataSize} \frac{\ab{\Reward_i}}{\Propensity_i} \norm{\Predictor_{\Round}(\Context_i)}^2
        } .
\label{eq:emp_risk_quad_upper_bound_decrease}
\end{equation}
From this, it is clear that an optimal predictor at round $\Round$ is
\begin{equation}
    \Predictor_{\Round}^\star
    \in \argmax_{\Predictor \in \Predictors} \,
    \frac{
        \left( \frac{1}{\DataSize} \sum_{i=1}^{\DataSize} \frac{\Reward_i}{\Propensity_i} \EnsemblePolicy[\Round-1](\Action_i \| \Context_i) (\ActionOneHot_i - \EnsemblePolicy[\Round-1](\Context_i))\T \Predictor(\Context_i) \right)^2
    }{
         \frac{1}{\DataSize} \sum_{i=1}^{\DataSize} \frac{\ab{\Reward_i}}{\Propensity_i} \norm{\Predictor(\Context_i)}^2
    } .
\label{eq:optimal_base_predictor}
\end{equation}
Since $\Predictor_{\Round}^\star$ is invariant to scaling, we can fix its scale by constraining $\frac{1}{\DataSize} \sum_{i=1}^{\DataSize} \frac{\ab{\Reward_i}}{\Propensity_i} \norm{\Predictor(\Context_i)}^2 = \NormConst$, for $\NormConst > 0$. We can then simply maximize the magnitude of the numerator in \cref{eq:optimal_base_predictor}, subject to this constraint. This results in the base learning objective:
\begin{equation}
    \max_{\Predictor \in \Predictors} \, \ab{
        \frac{1}{\DataSize} \sum_{i=1}^{\DataSize} \frac{\Reward_i}{\Propensity_i} \EnsemblePolicy[\Round-1](\Action_i \| \Context_i) (\ActionOneHot_i - \EnsemblePolicy[\Round-1](\Context_i))\T \Predictor(\Context_i)
    }
    \quad \text{s.t.} \quad \frac{1}{\DataSize} \sum_{i=1}^{\DataSize} \frac{\ab{\Reward_i}}{\Propensity_i} \norm{\Predictor(\Context_i)}^2 =  \NormConst .
\label{eq:base_learning_obj}
\end{equation}
This optimization becomes line 3 of \cref{alg:boosted_policy_learning}. \cref{eq:optimal_ensemble_weight} becomes line 4.

\subsection{Derivation of \cref{alg:surrogate_boosted_policy_learning} (BOPL-S)}
\label{sec:derivation_surrogate_boosted_policy_learning}

Recall from \cref{sec:surrogate_objective} that we use a composite of the original loss function, $\Loss_i$, and the surrogate loss function, $\LogLoss_i$, to upper bound the empirical risk:
\begin{equation}
    \EmpRisk(\EnsemblePolicy[\Round], \Data)
    \leq \SurrogateEmpRisk(\EnsemblePolicy[\Round], \Data)
    = \frac{1}{\DataSize} \sum_{i=1}^{\DataSize} \1\{ \Reward_i < 0 \} \Loss_i(\Ensemble[\Round]) + \1\{ \Reward_i \geq 0 \} \LogLoss_i(\Ensemble[\Round]) .
\end{equation}
Since both $\Loss_i$ and $\LogLoss_i$ are smooth (see \cref{sec:smoothness}), we can construct a recursive upper bound on the righthand side that isolates $\EnsembleWeight_{\Round}$ and $\Predictor_{\Round}$. Recall the upper bound for $\Loss_i$ given in \cref{eq:loss_smooth_upper_bound}. Further, using \cref{prop:log_loss_function_smooth,lem:smooth_upper_bound}, and noting the gradient,
\begin{equation}
    \grad \LogLoss_i(\Ensemble[\Round])
    = -\frac{\Reward_i}{\Propensity_i} (\ActionOneHot_i - \EnsemblePolicy[\Round](\Context_i)) ,
\label{eq:log_loss_gradient_wrt_ensemble}
\end{equation}
we have that
\begin{equation}
    \LogLoss_i(\Ensemble[\Round])
    \leq \LogLoss_i(\Ensemble[\Round-1])
        - \frac{\Reward_i}{\Propensity_i} (\ActionOneHot_i - \EnsemblePolicy[\Round-1](\Context_i))\T (\EnsembleWeight_{\Round} \Predictor_{\Round}(\Context_i))
        + \frac{\ab{\Reward_i}}{2 \Propensity_i} \norm{\EnsembleWeight_{\Round} \Predictor_{\Round}(\Context_i)}^2 .
\label{eq:log_loss_smooth_upper_bound}
\end{equation}
Thus, combining \cref{eq:loss_smooth_upper_bound,eq:log_loss_smooth_upper_bound}, we obtain a recursive upper bound,
\begin{align}
    \SurrogateEmpRisk(\EnsemblePolicy[\Round], \Data)
    &\leq \SurrogateEmpRisk(\EnsemblePolicy[\Round-1], \Data)
        - \frac{1}{\DataSize} \sum_{i=1}^{\DataSize} \Bigg[
            \frac{\Reward_i \GradSwitch_{i,t}}{\Propensity_i} (\ActionOneHot_i - \EnsemblePolicy[\Round-1](\Context_i))\T (\EnsembleWeight_{\Round} \Predictor_{\Round}(\Context_i))
            - \frac{\ab{\Reward_i} \Smoothness_i }{2 \Propensity_i} \norm{\EnsembleWeight_{\Round} \Predictor_{\Round}(\Context_i)}^2
        \Bigg], \\
    \text{where}\quad
    \GradSwitch_{i,t} &\defeq
        \begin{cases}
        \EnsemblePolicy[\Round-1](\Action_i \| \Context_i) & \text{if } \Reward_i < 0 , \\
        1 & \text{if } \Reward_i \geq 0 ;
        \end{cases}
    \eqand
    \Smoothness_i \defeq
        \begin{cases}
        \frac{1}{2} & \text{if } \Reward_i < 0 , \\
        1 & \text{if } \Reward_i \geq 0 .
        \end{cases}
\end{align}

The rest of the derivation proceeds similarly to \cref{sec:derivation_boosted_policy_learning}. Solving for the optimal ensemble weight, we get
\begin{equation}
    \EnsembleWeight_{\Round}^\star \defeq
    \frac{
        \frac{1}{\DataSize} \sum_{i=1}^{\DataSize} \frac{\Reward_i \GradSwitch_{i,t}}{\Propensity_i} (\ActionOneHot_i - \EnsemblePolicy[\Round-1](\Context_i))\T \Predictor_{\Round}(\Context_i)
    }{
        \frac{1}{\DataSize} \sum_{i=1}^{\DataSize} \frac{\ab{\Reward_i} \Smoothness_i}{\Propensity_i} \norm{\Predictor_{\Round}(\Context_i)}^2
    } .
\label{eq:surrogate_optimal_ensemble_weight}
\end{equation}
Then, using this value in the upper bound, we get
\begin{equation}
    \SurrogateEmpRisk(\EnsemblePolicy[\Round], \Data)
    \leq \SurrogateEmpRisk(\EnsemblePolicy[\Round-1], \Data)
        -\frac{
            \left( \frac{1}{\DataSize} \sum_{i=1}^{\DataSize} \frac{\Reward_i \GradSwitch_{i,t}}{\Propensity_i} (\ActionOneHot_i - \EnsemblePolicy[\Round-1](\Context_i))\T \Predictor_{\Round}(\Context_i) \right)^2
        }{
             \frac{2}{\DataSize} \sum_{i=1}^{\DataSize} \frac{\ab{\Reward_i} \Smoothness_i}{\Propensity_i} \norm{\Predictor_{\Round}(\Context_i)}^2
        } .
\label{eq:surrogate_emp_risk_quad_upper_bound_decrease}
\end{equation}
Therefore, an optimal predictor at round $\Round$ is given by
\begin{equation}
    \Predictor_{\Round}^\star
    \in \argmax_{\Predictor \in \Predictors} \,
    \frac{
        \left( \frac{1}{\DataSize} \sum_{i=1}^{\DataSize} \frac{\Reward_i \GradSwitch_{i,t}}{\Propensity_i} (\ActionOneHot_i - \EnsemblePolicy[\Round-1](\Context_i))\T \Predictor(\Context_i) \right)^2
    }{
         \frac{2}{\DataSize} \sum_{i=1}^{\DataSize} \frac{\ab{\Reward_i} \Smoothness_i}{\Propensity_i} \norm{\Predictor(\Context_i)}^2
    } .
\label{eq:surrogate_optimal_base_predictor}
\end{equation}
Ignoring the $1/2$ scaling (which does not affect the argmax), and recognizing that $\Predictor_{\Round}^\star$ is scale-invariant, we obtain the following constrained optimization problem for the base learner:
\begin{equation}
    \max_{\Predictor \in \Predictors} \, \ab{
        \frac{1}{\DataSize} \sum_{i=1}^{\DataSize} \frac{\Reward_i \GradSwitch_{i,t}}{\Propensity_i} (\ActionOneHot_i - \EnsemblePolicy[\Round-1](\Context_i))\T \Predictor(\Context_i)
    }
    \quad \text{s.t.} \quad \frac{1}{\DataSize} \sum_{i=1}^{\DataSize} \frac{\ab{\Reward_i} \Smoothness_i}{\Propensity_i} \norm{\Predictor(\Context_i)}^2 =  \NormConst .
\end{equation}
This becomes line 3 of \cref{alg:surrogate_boosted_policy_learning}, and \cref{eq:surrogate_optimal_ensemble_weight} becomes line 4.

\section{EXCESS RISK ANALYSIS}
\label{sec:excess_risk_analysis}

This appendix provides the proof of our excess empirical risk bound (\cref{th:emp_risk_bound}), and then shows how it can be combined with concentration and uniform convergence to bound the excess population risk.

\subsection{Proof of \cref{th:emp_risk_bound}}
\label{sec:proof_emp_risk_bound}

When we substitute the algorithm's constraint that $\frac{1}{\DataSize} \sum_{i=1}^{\DataSize} \frac{\ab{\Reward_i}}{\Propensity_i} \norm{\Predictor(\Context_i)}^2 = \NormConst$ into \cref{eq:emp_risk_quad_upper_bound_decrease}, we obtain
\begin{align}
    \EmpRisk(\EnsemblePolicy[\Round], \Data)
    &\leq \EmpRisk(\EnsemblePolicy[\Round-1], \Data) - \frac{1}{\NormConst} \left( \frac{1}{\DataSize} \sum_{i=1}^{\DataSize} \frac{\Reward_i}{\Propensity_i} \EnsemblePolicy[\Round-1](\Action_i \| \Context_i) (\ActionOneHot_i - \EnsemblePolicy[\Round-1](\Context_i))\T \Predictor_{\Round}(\Context_i) \right)^2 \\
    &= \EmpRisk(\EnsemblePolicy[\Round-1], \Data) - \frac{\NormConst}{4} \EnsembleWeight_{\Round}^2 ,
\end{align}
in which we reduce the righthand expression using \cref{eq:optimal_ensemble_weight}. This bound is recursive, depending on the empirical risk of the previous ensemble policy, $\EnsemblePolicy[\Round-1]$. The base case is $\EmpRisk(\EnsemblePolicy[0], \Data)$. Unraveling the recursion from round $\Rounds$, we get
\begin{align}
    \EmpRisk(\EnsemblePolicy, \Data)
    \leq \EmpRisk(\EnsemblePolicy[\Rounds-1], \Data) - \frac{\NormConst}{4} \EnsembleWeight_{\Rounds}^2
    \leq \EmpRisk(\EnsemblePolicy[0], \Data) - \frac{\NormConst}{4} \sum_{\Round=1}^{\Rounds} \EnsembleWeight_{\Round}^2 .
\label{eq:quad_emp_risk_bound}
\end{align}
Subtracting $\OptEmpRisk$ from both sides of the inequality, we get
\begin{equation}
    \EmpRisk(\EnsemblePolicy, \Data) - \OptEmpRisk
    \leq \EmpRisk(\EnsemblePolicy[0], \Data) - \OptEmpRisk - \frac{\NormConst}{4} \sum_{\Round=1}^{\Rounds} \EnsembleWeight_{\Round}^2
    = \ExcessEmpRisk_0 - \frac{\NormConst}{4} \sum_{\Round=1}^{\Rounds} \EnsembleWeight_{\Round}^2 .
\label{eq:quad_excess_emp_risk_bound}
\end{equation}
Note that $\ExcessEmpRisk_0$ is nonnegative, by definition of $\OptEmpRisk$. Therefore, using the identity $c (1 - z) \leq c e^{-z}$, for all $c \in \Reals_+$ and $z \in \Reals$, we have that
\begin{equation}
    (\ref{eq:quad_excess_emp_risk_bound})
    = \ExcessEmpRisk_0 \left(
        1 - \frac{\NormConst}{4 \ExcessEmpRisk_0} \sum_{\Round=1}^{\Rounds} \EnsembleWeight_{\Round}^2
    \right)
    \leq \ExcessEmpRisk_0 \exp\left(
        - \frac{\NormConst}{4 \ExcessEmpRisk_0} \sum_{\Round=1}^{\Rounds} \EnsembleWeight_{\Round}^2
    \right) ,
\label{eq:exp_emp_risk_bound}
\end{equation}
which completes the proof.

\subsection{Excess Population Risk Bound}
\label{sec:excess_risk_bound}

We now explain how to relate \cref{th:emp_risk_bound} to an upper bound on BOPL's excess risk relative to an optimal policy. Let $\OptPolicy \in \argmin_{\Policy} \Risk(\Policy)$ denote an optimal policy (i.e., risk minimizer), and let $\OptRisk \defeq \Risk(\OptPolicy)$ denote its corresponding risk. Recall that $\EmpOptPolicy \in \argmin_{\Policy} \EmpRisk(\Policy, \Data)$ is an \emph{empirically} optimal policy (i.e., empirical risk minimizer) for a given dataset, $\Data$, and $\OptEmpRisk \defeq \EmpRisk(\EmpOptPolicy)$ is its corresponding empirical risk. Using these definitions, the excess risk can be expressed as $\Risk(\EnsemblePolicy) - \OptRisk$, and the excess empirical risk is $\EmpRisk(\EnsemblePolicy, \Data) - \OptEmpRisk$.

Via simple arithmetic, we can expand the excess risk into several terms:
\begin{align}
    \underbrace{\Risk(\EnsemblePolicy) - \OptRisk}_{\text{excess risk}}
    &= \Risk(\EnsemblePolicy) - \EmpRisk(\EnsemblePolicy, \Data) + \EmpRisk(\EnsemblePolicy, \Data) - \OptEmpRisk + \OptEmpRisk - \EmpRisk(\OptPolicy, \Data) + \EmpRisk(\OptPolicy, \Data) - \OptRisk \\
    &\leq \underbrace{\Risk(\EnsemblePolicy) - \EmpRisk(\EnsemblePolicy, \Data)}_{\text{generalization error}} + \underbrace{\EmpRisk(\EnsemblePolicy, \Data) - \OptEmpRisk}_{\text{excess empirical risk}} + \underbrace{\EmpRisk(\OptPolicy, \Data) - \OptRisk}_{\text{estimation error}} .
\label{eq:excess_risk_decomp}
\end{align}
Starting on the right, the difference $\EmpRisk(\OptPolicy, \Data) - \OptRisk$ captures our ability to \emph{estimate} the risk of a policy (in this case, the optimal policy) using a finite sample of data. Similarly, on the left, $\Risk(\EnsemblePolicy) - \EmpRisk(\EnsemblePolicy, \Data)$ captures the learning algorithm's ability to \emph{generalize} from finite data, as quantified by the difference of the risk and empirical risk---the latter of which is being optimized. Finally, the middle difference is the excess empirical risk---which is what we upper-bound in \cref{th:emp_risk_bound}.

Thus, to upper-bound the excess risk, we must bound the estimation and generalization errors. Estimation and generalization have been studied extensively in statistical learning theory, so we have many tools at our disposal to upper-bound those error terms. In the following, we provide an example bound---which is by no means optimal, but is merely meant to illustrate how to apply existing theory to complete the picture.

For simplicity, we will assume that rewards are bounded in $[-1, 1]$, and that the logged propensities are uniformly lower-bounded by some positive constant, $\LogPolicy(\dummyvar \| \dummyvar) \geq \MinPropensity > 0$. With these assumptions, we have that the random variable $\AltLossFunc(\Policy, \Context, \Action, \Propensity, \Reward) \defeq -\frac{\Reward}{\Propensity} \Policy(\Action \| \Context)$ is almost-surely bounded in $[-\MinPropensity^{-1}, \MinPropensity^{-1}]$. Note that $\Risk(\Policy) = \Ep[ \AltLossFunc(\Policy, \Context, \Action, \Propensity, \Reward) ]$ and $\EmpRisk(\Policy, \Data) \defeq \frac{1}{\DataSize} \sum_{i=1}^{\DataSize} \AltLossFunc(\Policy_i, \Context_i, \Action_i, \Propensity_i, \Reward_i)$.

Accordingly, since $\EmpRisk(\Policy, \Data)$ is just an average of bounded, i.i.d.\ random variables, we can use any applicable concentration inequality to upper-bound the estimation error. For example, Hoeffding's inequality \citep{hoeffding:jasa63} yields
\begin{equation}
    \Pr_{\Data}\left\{ \EmpRisk(\OptPolicy, \Data) - \Risk(\OptPolicy) \geq \epsilon \right\}
    \leq \exp\left(
        -\frac{\DataSize \MinPropensity^2 \epsilon^2}{2}
    \right) ;
\end{equation}
so, with probability at least $1 - \delta/2$ over draws of $\Data$,
\begin{equation}
    \EmpRisk(\OptPolicy, \Data) - \Risk(\OptPolicy)
    \leq \frac{1}{\MinPropensity} \sqrt{ \frac{2}{\DataSize} \ln\frac{2}{\delta} } .
\label{eq:estimation_error_bound}
\end{equation}

There are many tools available to bound the generalization error of a learning algorithm. We will approach it from the perspective of \emph{uniform convergence}---that is, we will show that, with high probability, the generalization error of any $\EnsemblePolicy \in \EnsemblePolicies$ is vanishing in $\DataSize$. To do so, we will leverage a standard Rademacher complexity-based generalization bound. For a generic real-valued function class, $\cG$, its \emph{Rademacher complexity} is
\begin{equation}
    \RademacherComp(\cG) \defeq \Ep\left[ \sup_{g \in \cG} \frac{1}{\DataSize} \sum_{i=1}^{\DataSize} \RademacherVar_i g(z_i) \right] ,
\label{eq:rademacher_complexity}
\end{equation}
where $z_1, \dots, z_{\DataSize}$ are i.i.d.\ draws from an arbitrary distribution, and $\RademacherVar_1, \dots, \RademacherVar_{\DataSize}$ are independent \emph{Rademacher variables}, which are uniformly distributed over $\TwoClass$. Let
\begin{equation}
    \AltLosses \defeq \AltLossFunc \circ \EnsemblePolicies = \{ (\Context, \Action, \Propensity, \Reward) \mapsto \AltLossFunc(\EnsemblePolicy, \Context, \Action, \Propensity, \Reward) : \EnsemblePolicy \in \EnsemblePolicies \}
\label{eq:alt_losses}
\end{equation}
denote the composition of $\AltLossFunc$ and $\EnsemblePolicies$, where $\AltLoss \in \AltLosses$ is a member of the class. Leveraging \citep[Theorem 3.3]{mohri:book12}, we have that, for any $\delta' \in (0, 1)$ (to be defined later), with probability at least $1 - \delta'$,
\begin{align}
    \Risk(\EnsemblePolicy) - \EmpRisk(\EnsemblePolicy, \Data)
    &\leq \sup_{\EnsemblePolicy \in \EnsemblePolicies} \Risk(\EnsemblePolicy) - \EmpRisk(\EnsemblePolicy, \Data) \\
    &= \sup_{\AltLoss \in \AltLosses} \, \Ep[ \AltLoss(\Context, \Action, \Propensity, \Reward) ] - \frac{1}{\DataSize} \sum_{i=1}^{\DataSize} \AltLoss(\Context_i, \Action_i, \Propensity_i, \Reward_i) \\
    &\leq 2 \, \RademacherComp(\AltLosses) + \frac{1}{\MinPropensity} \sqrt{\frac{2}{\DataSize} \ln\frac{1}{\delta'}} .
\label{eq:rademacher_risk_bound}
\end{align}
Note that we modified the bound to account for the range of $\AltLossFunc$. Taking $\delta' = \delta/2$, we have that \cref{eq:estimation_error_bound,eq:rademacher_risk_bound} hold with probability at least $1 - \delta$; and when combined with \cref{eq:emp_risk_bound}, we upper-bound the excess risk with high probability.

However, we can make the dependence on $\Predictors$ more explicit by focusing on the Rademacher term. First, we note that $\AltLossFunc$ is $\MinPropensity^{-1}$-Lipschitz with respect to the $2$-norm of the ensemble predictor output---that is, for any $\Ensemble, \Ensemble' \in \Ensembles$ (with associated policies, $\EnsemblePolicy, \EnsemblePolicy'$), $\Context \in \Contexts$, $\Action \in \Actions$, $\Propensity \in [\MinPropensity, 1]$ and $\Reward \in [-1, 1]$,
\begin{equation}
    \ab{\AltLossFunc(\Policy, \Context, \Action, \Propensity, \Reward) - \AltLossFunc(\Policy', \Context, \Action, \Propensity, \Reward)}
    \leq \frac{1}{\MinPropensity} \twonorm{\Ensemble(\Context) - \Ensemble'(\Context)} .
\label{eq:alt_loss_lipschitz}
\end{equation}
This is readily verified by noting that $\ab{\Reward / \Propensity} \leq \MinPropensity^{-1}$, and using the fact that the softmax function (for $\InvTemp = 1$) is $1$-Lipschitz \citep[Proposition 4]{gao:arxiv17}. Therefore, using a vector-valued extension of Talagrand's contraction lemma \citep[Corollary 1]{maurer:alt16}, we have that
\begin{align}
    \RademacherComp(\AltLosses)
    &= \Ep\left[ \sup_{\AltLoss \in \AltLosses} \frac{1}{\DataSize} \sum_{i=1}^{\DataSize} \RademacherVar_i \AltLoss(\Context_i, \Action_i, \Propensity_i, \Reward_i) \right] \\
    &= \Ep\left[ \sup_{\EnsemblePolicy \in \EnsemblePolicies} \frac{1}{\DataSize} \sum_{i=1}^{\DataSize} \RademacherVar_i \AltLossFunc(\EnsemblePolicy, \Context_i, \Action_i, \Propensity_i, \Reward_i) \right] \\
    &\leq \frac{\sqrt{2}}{\MinPropensity} \underbrace{
        \Ep\left[ \sup_{\Ensemble \in \Ensembles} \frac{1}{\DataSize} \sum_{i=1}^{\DataSize} \sum_{\Action \in \Actions} \RademacherVar_{i, \Action} \Ensemble(\Context_i, \Action) \right]
        }_{\defeq \, \RademacherComp(\Ensembles)} .
\label{eq:vector_valued_contraction}
\end{align}
We are left with the (vector-valued) Rademacher complexity of the class of ensemble predictors, $\Ensembles$. To get to the complexity of $\Predictors$, we will appeal to well known results for convex combinations of hypotheses.

Unfortunately, softmax ensemble policies are not convex combinations of predictors. Nonetheless, we can transform $\EnsemblePolicies$ into a class of convex ensembles, so that we can leverage existing Rademacher bounds. First, we will assume that $\Predictors$ is symmetric. By implication, we can assume that every ensemble weight, $\EnsembleWeight_{\Round}$, is nonnegative---since any pair, $(\Predictor, \EnsembleWeight) : \Predictor \in \Predictors, \, \EnsembleWeight < 0$, has a corresponding $\Predictor' \in \Predictors$ such that $\EnsembleWeight \Predictor = (-\EnsembleWeight) \Predictor'$. Then, we will temporarily assume that the ensemble weights have $1$-norm bounded by some constant, $\MaxWeightNorm > 0$, which will allow us to normalize the weights and thereby obtain the Rademacher complexity of convex ensembles, scaled by $\MaxWeightNorm$. Finally, we construct a covering of all $\MaxWeightNorm$, which allows us to obtain high-probability bounds that hold for all $\MaxWeightNorm$ simultaneously.

Let
\begin{equation}
    \Ensembles^{\MaxWeightNorm} \defeq \Bigg\{
        (\Context, \Action) \mapsto \Ensemble(\Context, \Action) :
        \forall \Round, \,
        \Predictor_{\Round} \in \Predictors, \,
        \EnsembleWeight_{\Round} \in \Reals_+ ; \,
        \sum_{\Round=1}^{\Rounds} \EnsembleWeight_{\Round} \leq \MaxWeightNorm
    \Bigg\}
\label{eq:bounded_ensembles}
\end{equation}
denote the set of ensembles for $\Predictors$ with nonnegative weights, whose sum is upper-bounded by $\MaxWeightNorm$. Further, let
\begin{equation}
    \ConvexEnsembles \defeq \Bigg\{
        (\Context, \Action) \mapsto \Ensemble(\Context, \Action) :
        \forall \Round, \,
        \Predictor_{\Round} \in \Predictors, \,
        \EnsembleWeight_{\Round} \in [0, 1] ; \,
        \sum_{\Round=1}^{\Rounds} \EnsembleWeight_{\Round} = 1
    \Bigg\}
\label{eq:convex_ensembles}
\end{equation}
denote the set of convex ensembles for $\Predictors$. For simplicity, we will write $\onenorm{\EnsembleWeights} \defeq \sum_{\Round=1}^{\Rounds} \ab{\EnsembleWeight_{\Round}}$, where $\EnsembleWeights \defeq (\EnsembleWeight_1, \dots, \EnsembleWeight_{\Rounds})$ denotes a vector of ensemble weights, whose length, $\Rounds$, should be clear from context. (The absolute value is unnecessary when all weights are nonnegative, but we include it for correctness.) Observe that, for any $\Ensemble \in \Ensembles$, with weights $\EnsembleWeights$, there exists a $\ConvexEnsemble \in \ConvexEnsembles$ such that $\ConvexEnsemble(\dummyvar, \dummyvar) = \frac{\Ensemble(\dummyvar, \dummyvar)}{\onenorm{\EnsembleWeights}}$ (when $\Predictors$ is assumed to be symmetric). Thus,
\begin{align}
    \RademacherComp(\Ensembles^{\MaxWeightNorm})
    &= \Ep\left[ \sup_{\Ensemble \in \Ensembles} \frac{\onenorm{\EnsembleWeights}}{\DataSize} \sum_{i=1}^{\DataSize} \sum_{\Action \in \Actions} \RademacherVar_{i, \Action} \frac{\Ensemble(\Context_i, \Action)}{\onenorm{\EnsembleWeights}} \right] \\
    &\leq \Ep\left[ \sup_{\Ensemble \in \Ensembles} \frac{\MaxWeightNorm}{\DataSize} \sum_{i=1}^{\DataSize} \sum_{\Action \in \Actions} \RademacherVar_{i, \Action} \frac{\Ensemble(\Context_i, \Action)}{\onenorm{\EnsembleWeights}} \right] \\
    &= \MaxWeightNorm \Ep\left[ \sup_{\ConvexEnsemble \in \ConvexEnsembles} \frac{1}{\DataSize} \sum_{i=1}^{\DataSize} \sum_{\Action \in \Actions} \RademacherVar_{i, \Action} \ConvexEnsemble(\Context_i, \Action) \right] \\
    &= \MaxWeightNorm \, \RademacherComp(\ConvexEnsembles) .
\end{align}
Having reduced $\RademacherComp(\Ensembles^{\MaxWeightNorm})$ to a function of $\RademacherComp(\ConvexEnsembles)$, we can leverage a classical result \citep[Theorem 12]{bartlett:jmlr03}, which states (among other things): (1) for classes $\cG$ and $\cH$, if $\cG \subseteq \cH$, then $\RademacherComp(\cG) \leq \RademacherComp(\cH)$; (2) for the \emph{convex hull} of $\cH$, denoted $\conv(\cH)$, we have $\RademacherComp(\conv(\cH)) = \RademacherComp(\cH)$. Therefore, since $\ConvexEnsembles \subseteq \conv(\Predictors)$, we have that $\RademacherComp(\ConvexEnsembles) \leq \RademacherComp(\conv(\Predictors)) = \RademacherComp(\Predictors)$; and thus, $\RademacherComp(\Ensembles^{\MaxWeightNorm}) \leq \MaxWeightNorm \, \RademacherComp(\Predictors)$.

Finally, having related $\RademacherComp(\AltLosses)$ to $\RademacherComp(\Predictors)$ when the ensemble weights are constrained to a specific sum, $\MaxWeightNorm$, we want \cref{eq:rademacher_risk_bound} to hold, with probability at least $1 - \delta/2$, for \emph{all} $\MaxWeightNorm$ simultaneously. For $j = 0, 1, 2, \dots$, let $\MaxWeightNorm_j \defeq 2^j$ and $\delta_j \defeq \frac{\delta}{4} \MaxWeightNorm_j^{-1}$. Observe that $\delta_j$ forms a geometric series, and $\sum_{j=0}^{\infty} \delta_j = \frac{\delta}{4} \sum_{j=0}^{\infty} 2^{-j} = \delta/2$. Thus, assigning probability $\delta_j$ to each $\MaxWeightNorm_j$, we have that, with probability at least $1 - \delta/2$, for all $j$ simultaneously,
\begin{align}
    \sup_{\EnsemblePolicy \in \EnsemblePolicies^{\MaxWeightNorm_j}} \Risk(\EnsemblePolicy) - \EmpRisk(\EnsemblePolicy, \Data)
    &\leq 2 \, \RademacherComp(\AltLosses^{\MaxWeightNorm_j}) + \frac{1}{\MinPropensity} \sqrt{\frac{2}{\DataSize} \ln\frac{1}{\delta_j}} \\
    &\leq \frac{\sqrt{8}}{\MinPropensity} \, \RademacherComp(\Ensembles^{\MaxWeightNorm_j}) + \frac{1}{\MinPropensity} \sqrt{\frac{2}{\DataSize} \ln\frac{1}{\delta_j}} \\
    &\leq \frac{\sqrt{8}}{\MinPropensity} \MaxWeightNorm_j \RademacherComp(\Predictors) + \frac{1}{\MinPropensity} \sqrt{\frac{2}{\DataSize} \ln\frac{1}{\delta_j}} ,
\label{eq:rademacher_risk_bound_maxweightnorm}
\end{align}
where we use superscript $\MaxWeightNorm_j$ in $\EnsemblePolicies^{\MaxWeightNorm_j}$ and $\AltLosses^{\MaxWeightNorm_j}$ to indicate that the underlying class of ensembles has weights bounded accordingly. What remains is to pick a value of $j$ for the learned ensemble policy, $\EnsemblePolicy$. Taking
\begin{equation}
    j^\star \defeq \ceil{ (\ln 2)^{-1} \ln\max\{\onenorm{\EnsembleWeights}, 1\} } ,
\label{eq:j_star}
\end{equation}
we have that
\begin{equation}
    \MaxWeightNorm_{j^\star}
    = 2^{\ceil{ (\ln 2)^{-1} \ln\max\{\onenorm{\EnsembleWeights}, 1\} }}
    \geq 2^{ (\ln 2)^{-1} \ln\max\{\onenorm{\EnsembleWeights}, 1\} }
    = \max\{\onenorm{\EnsembleWeights}, 1\}
    \geq \onenorm{\EnsembleWeights} ;
\label{eq:maxweightnorm_j_star_lb}
\end{equation}
meaning, the learned ensemble, $\Ensemble$, is contained in the class $\Ensembles^{\MaxWeightNorm_{j^\star}}$, so the bound for $j^\star$ is valid for $\EnsemblePolicy$. Further,
\begin{equation}
    \MaxWeightNorm_{j^\star}
    \leq 2^{ (\ln 2)^{-1} \ln\max\{\onenorm{\EnsembleWeights}, 1\} + 1 }
    = 2 \max\{\onenorm{\EnsembleWeights}, 1\} ,
\label{eq:maxweightnorm_j_star_ub}
\end{equation}
and
\begin{equation}
    \delta_{j^\star}^{-1}
    = \frac{4}{\delta} \MaxWeightNorm_{j^\star}
    \leq \frac{8}{\delta} \max\{\onenorm{\EnsembleWeights}, 1\} .
\label{eq:delta_j_star_lb}
\end{equation}
Putting it all together, we have that, with probability at least $1 - \delta/2$,
\begin{equation}
    \Risk(\EnsemblePolicy) - \EmpRisk(\EnsemblePolicy, \Data)
    \leq \frac{\sqrt{32}}{\MinPropensity} \max\{\onenorm{\EnsembleWeights}, 1\} \, \RademacherComp(\Predictors)
        + \frac{1}{\MinPropensity} \sqrt{\frac{2}{\DataSize} \ln\left( \frac{8}{\delta} \max\{\onenorm{\EnsembleWeights}, 1\} \right)} .
\label{eq:rademacher_risk_bound_final}
\end{equation}

Combining \cref{eq:emp_risk_bound,eq:estimation_error_bound,eq:rademacher_risk_bound_final}, we obtain a full characterization of the excess risk that holds with probability at least $1 - \delta$:
\begin{align}
    \underbrace{\Risk(\EnsemblePolicy) - \OptRisk}_{\text{excess risk}}
    ~~\leq
    &\quad \underbrace{
        \ExcessEmpRisk_0 \exp\left(
            - \frac{\NormConst}{4 \ExcessEmpRisk_0} \sum_{\Round=1}^{\Rounds} \EnsembleWeight_{\Round}^2
        \right)
    }_{\text{excess empirical risk}} \\
    &+ \underbrace{
        \frac{\sqrt{32}}{\MinPropensity} \max\{\onenorm{\EnsembleWeights}, 1\} \, \RademacherComp(\Predictors)
        + \frac{1}{\MinPropensity} \sqrt{\frac{2}{\DataSize} \ln\left( \frac{8}{\delta} \max\{\onenorm{\EnsembleWeights}, 1\} \right)}
    }_{\text{generalization error}} \\
    &+ \underbrace{
        \frac{1}{\MinPropensity} \sqrt{ \frac{2}{\DataSize} \ln\frac{2}{\delta} }
    }_{\text{est.\ error}} .
\label{eq:full_excess_risk_bound}
\end{align}
Note that the bound is stated in terms of the Rademacher complexity of the predictor class, $\Predictors$. For many useful hypothesis classes (such as decision stumps and trees), the Rademacher complexity vanishes at rate $\BigO(\DataSize^{-1/2})$. In such cases, if the excess empirical risk is small, then the excess risk also vanishes at rate $\BigO(\DataSize^{-1/2})$.

\begin{remark}
Since $\OptPolicy$ and $\EmpOptPolicy$ are defined via \emph{unconstrained} minimizations over all valid policies (not just those contained in $\EnsemblePolicies$), there could be an implicit gap between the minimum attainable risk and the risk of the best ensemble policy, $\OptEnsemblePolicy \in \argmin_{\EnsemblePolicy \in \EnsemblePolicies} \Risk(\EnsemblePolicy)$. We can make this gap explicit by modifying our excess risk decomposition: 
\begin{equation}
    \underbrace{\Risk(\EnsemblePolicy) - \OptRisk}_{\text{excess risk}}
    \leq \underbrace{\Risk(\EnsemblePolicy) - \EmpRisk(\EnsemblePolicy, \Data)}_{\text{generalization error}}
    + \underbrace{\EmpRisk(\EnsemblePolicy, \Data) - \OptEmpRisk}_{\text{excess empirical risk}}
    + \underbrace{\EmpRisk(\OptEnsemblePolicy, \Data) - \Risk(\OptEnsemblePolicy)}_{\text{estimation error}}
    + \underbrace{\Risk(\OptEnsemblePolicy) - \OptRisk}_{\text{approximation error}} .
\label{eq:alt_excess_risk_decomp}
\end{equation}
The new term---the so-called \emph{approximation error} \citep{bottou:nips07}---measures how well $\EnsemblePolicies$ fits the distribution. Unfortunately, the approximation error is unknowable in all but trivial cases, so it is typically assumed to be some small constant. The rest of the bound, when instantiated with the above analysis, would work out the same. Thus, in this case, there does not appear to be any advantage to making the approximation error explicit.
\end{remark}

\section{BASE LEARNING REDUCTIONS}
\label{sec:base_learning_reductions}

In this appendix, we derive base learners for two classes of predictors: (nonlinear) real-valued functions and binary classifiers. In both cases, we assume that the class, $\Predictors$, is \emph{symmetric}; meaning, for every $\Predictor \in \Predictors$, its negation, $-\Predictor$, is also in $\Predictors$.

For brevity, we focus on \cref{alg:boosted_policy_learning}, though it is straightforward to adapt the base learners for \cref{alg:surrogate_boosted_policy_learning}.

\subsection{Boosting via Regression}
\label{sec:boosting_via_regression}

Assume that $\Predictors$ is a symmetric class of real-valued functions, $\Predictors \subseteq \{ \Contexts \times \Actions \to \Reals \}$. Since an ensemble, $\Ensemble$, is a linear combination of predictors, it is important that $\Predictors$ is a nonlinear function class (such as regression trees), so that $\Ensemble$ can have greater expressive power than its constituents. (If the predictors were linear functions, then $\Ensemble$ would still be a linear function of its input.)

Recall the base learning objective in \cref{eq:base_learning_obj}. Since we assume that $\Predictors$ is symmetric, we can omit the absolute value; if some $\Predictor \in \Predictors$ is a minimizer of the expression inside the absolute value, then its negation, $-\Predictor$, is also in $\Predictors$. We then convert the constrained optimization problem to the following unconstrained one via Lagrangian relaxation:
\begin{equation}
    (\ref{eq:base_learning_obj})
    ~ = ~
    \max_{\Predictor \in \Predictors} \min_{\Lagrange \in \Reals} ~
        \frac{1}{\DataSize} \sum_{i=1}^{\DataSize} \frac{\Reward_i}{\Propensity_i} \EnsemblePolicy[\Round-1](\Action_i \| \Context_i) (\ActionOneHot_i - \EnsemblePolicy[\Round-1](\Context_i))\T \Predictor(\Context_i)
        - \Lagrange \left( \frac{1}{\DataSize} \sum_{i=1}^{\DataSize} \frac{\ab{\Reward_i}}{\Propensity_i} \norm{\Predictor(\Context_i)}^2 - \NormConst \right) .
\label{eq:base_learning_lagrangian}
\end{equation}
For every $\NormConst$, there exists a $\Lagrange$ that is optimal. The converse of this statement is that, for every $\Lagrange$, one can construct a $\NormConst$ for which $\Lagrange$ is optimal. Since $\NormConst$ is arbitrary, we can choose any $\Lagrange$ for the optimization. Without loss of generality, we take $\Lagrange = 1/2$, which results in
\begin{align}
    &~ \argmax_{\Predictor \in \Predictors} \sum_{i=1}^{\DataSize} \frac{\Reward_i}{\Propensity_i}
        \EnsemblePolicy[\Round-1](\Action_i \| \Context_i) (\ActionOneHot_i - \EnsemblePolicy[\Round-1](\Context_i))\T \Predictor(\Context_i) - \frac{\ab{\Reward_i}}{2 \Propensity_i} \norm{\Predictor(\Context_i)}^2 \\
    = &~ \argmax_{\Predictor \in \Predictors} \sum_{i=1}^{\DataSize} \frac{\ab{\Reward_i}}{\Propensity_i} \Big(
        \sgn(\Reward_i) \EnsemblePolicy[\Round-1](\Action_i \| \Context_i) (\ActionOneHot_i - \EnsemblePolicy[\Round-1](\Context_i))\T \Predictor(\Context_i) - \frac{1}{2} \norm{\Predictor(\Context_i)}^2
    \Big) \\
    = &~ \argmin_{\Predictor \in \Predictors} \sum_{i=1}^{\DataSize} \frac{\ab{\Reward_i}}{\Propensity_i} \norm{
        \sgn(\Reward_i) \EnsemblePolicy[\Round-1](\Action_i \| \Context_i) (\ActionOneHot_i - \EnsemblePolicy[\Round-1](\Context_i)) - \Predictor(\Context_i)
    }^2 .
\label{eq:base_learning_regression}
\end{align}
This is a weighted least-squares regression, which can be solved by a variety of off-the-shelf tools.

The resulting boosting algorithm is given in \cref{alg:boosted_policy_learning_regression}. Line 3 uses a given subroutine for solving weighted least-squares regression problems with the class $\Predictors$.

Since the regression base learner does not explicitly enforce the constraint in \cref{eq:base_learning_obj}, that $\frac{1}{\DataSize} \sum_{i=1}^{\DataSize} \frac{\ab{\Reward_i}}{\Propensity_i} \norm{\Predictor_{\Round}(\Context_i)}^2 = \NormConst$ for a given $\NormConst > 0$, one cannot immediately apply \cref{th:emp_risk_bound} to \cref{alg:boosted_policy_learning_regression}. However, recall that $\Lagrange$ implicitly ensures that $\frac{1}{\DataSize} \sum_{i=1}^{\DataSize} \frac{\ab{\Reward_i}}{\Propensity_i} \norm{\Predictor_{\Round}(\Context_i)}^2 = \NormConst'$ for some $\NormConst' > 0$; and if the base learning problem can be solved for $\NormConst'$, then it can be solved for any $\NormConst$ by simply rescaling the predictor, $\Predictor_{\Round} \gets \Predictor_{\Round} \sqrt{\NormConst / \NormConst'}$, which does not affect the predictor's optimality with respect to \cref{eq:optimal_base_predictor} if it is rescaled prior to computing $\EnsembleWeight_{\Round}$. Thus, given the output of the base learner (line 3), we compute $\NormConst'$, then perform the rescaling prior to line 4, thereby ensuring that \cref{th:emp_risk_bound} holds for $\NormConst$. All that being said, it is important to remember that this modification is not strictly necessary for the algorithm to work; only for \cref{th:emp_risk_bound} to hold.

\begin{algorithm}
    \caption{Boosted Off-Policy Learning via Regression}
    \label{alg:boosted_policy_learning_regression}
    \begin{algorithmic}[1]
    \Require{symmetric, real-valued class, $\Predictors$; solver for weighted least-squares; rounds, $\Rounds \geq 1$}
    \State $\Ensemble[0] \gets 0$
    \For{$\Round = 1, \dots, \Rounds$}
        \State $\Predictor_{\Round} \gets \argmin_{\Predictor \in \Predictors} \sum_{i=1}^{\DataSize} \sum_{\Action \in \Actions} \ImportanceWeight_i ( \PseudoLabel_{i,\Action} - \Predictor(\Context_i, \Action) )^2$
        \Comment{weighted least-squares}
        \Statex \quad\quad\quad with ~ $\ImportanceWeight_i \gets \frac{\ab{\Reward_i}}{\Propensity_i}$
        \Statex \quad\quad\quad \, and ~ $\PseudoLabel_{i, \Action} \gets \sgn(\Reward_i) \EnsemblePolicy[\Round-1](\Action_i \| \Context_i) ( \1\{ \Action = \Action_i \} - \EnsemblePolicy[\Round-1](\Action \| \Context_i) )$
        \State $\EnsembleWeight_{\Round} \gets \frac{
            2 \sum_{i=1}^{\DataSize} \frac{\Reward_i}{\Propensity_i} \EnsemblePolicy[\Round-1](\Action_i \| \Context_i) (\ActionOneHot_i - \EnsemblePolicy[\Round-1](\Context_i))\T \Predictor_{\Round}(\Context_i)
        }{
            \sum_{i=1}^{\DataSize} \frac{\ab{\Reward_i}}{\Propensity_i} \norm{\Predictor_{\Round}(\Context_i)}^2
        }$
        \State $\Ensemble[\Round] \gets \Ensemble[\Round-1] + \EnsembleWeight_{\Round} \Predictor_{\Round}$
    \EndFor
    \end{algorithmic}
\end{algorithm}

\subsection{Boosting via Binary Classification}
\label{sec:boosting_via_binary_classification}

The following reduction to binary classification is inspired by \citep[Section 7.4.3]{schapire:book12}. Assume that $\Predictors$ a symmetric class of binary classifiers, $\Predictors \subseteq \{ \Contexts \times \Actions \to \TwoClass \}$, such as decision stumps. For every example, $i \in \{1, \dots, \DataSize\}$, and action, $\Action \in \Actions$, we define a nonnegative weight,
\begin{equation}
    \ImportanceWeight_{i, \Action} \defeq \ab{
        \frac{\Reward_i}{\Propensity_i} \EnsemblePolicy[\Round-1](\Action_i \| \Context_i) ( \1\{ \Action = \Action_i \} - \EnsemblePolicy[\Round-1](\Action \| \Context_i) )
    } ,
\end{equation}
and a $\TwoClass$-valued pseudo-label,
\begin{equation}
    \PseudoLabel_{i, \Action} \defeq \sgn(\Reward_i) ( 2 \, \1\{ \Action = \Action_i \} - 1 )
    = \sgn\bigg( \frac{\Reward_i}{\Propensity_i} \EnsemblePolicy[\Round-1](\Action_i \| \Context_i) ( \1\{ \Action = \Action_i \} - \EnsemblePolicy[\Round-1](\Action \| \Context_i) ) \bigg) .
\end{equation}
(These variables are local to the current round, $\Round$, but we ignore this to simplify notation.) Using the righthand equivalence, we have that
\begin{equation}
    \ImportanceWeight_{i, \Action} \PseudoLabel_{i, \Action}
    = \frac{\Reward_i}{\Propensity_i} \EnsemblePolicy[\Round-1](\Action_i \| \Context_i) ( \1\{ \Action = \Action_i \} - \EnsemblePolicy[\Round-1](\Action \| \Context_i) ) .
\end{equation}
We also have that
\begin{equation}
    \1\{ \PseudoLabel_{i, \Action} \neq \Predictor(\Context_i, \Action) \}
    =
    \frac{1}{2} ( 1 - \PseudoLabel_{i, \Action} \Predictor(\Context_i, \Action) ) .
\end{equation}
Therefore, if we minimize the weighted classification error, we end up with the following equivalence:
\begin{align}
    &~ \argmin_{\Predictor \in \Predictors} \sum_{i=1}^{\DataSize} \sum_{\Action \in \Actions} \ImportanceWeight_{i, \Action} \1\{ \PseudoLabel_{i, \Action} \neq \Predictor(\Context_i, \Action) \} \\
    = &~ \argmin_{\Predictor \in \Predictors} \sum_{i=1}^{\DataSize} \sum_{\Action \in \Actions} \frac{\ImportanceWeight_{i, \Action}}{2} ( 1 - \PseudoLabel_{i, \Action} \Predictor(\Context_i, \Action) ) \\
    = &~ \argmax_{\Predictor \in \Predictors} \sum_{i=1}^{\DataSize} \sum_{\Action \in \Actions} \ImportanceWeight_{i, \Action} \PseudoLabel_{i, \Action} \Predictor(\Context_i, \Action) \\
    = &~ \argmax_{\Predictor \in \Predictors} \sum_{i=1}^{\DataSize} \sum_{\Action \in \Actions} \frac{\Reward_i}{\Propensity_i} \EnsemblePolicy[\Round-1](\Action_i \| \Context_i) ( \1\{ \Action = \Action_i \} - \EnsemblePolicy[\Round-1](\Action \| \Context_i) ) \Predictor(\Context_i, \Action) \\
    = &~ \argmax_{\Predictor \in \Predictors} \ab{ \frac{1}{\DataSize} \sum_{i=1}^{\DataSize} \frac{\Reward_i}{\Propensity_i} \EnsemblePolicy[\Round-1](\Action_i \| \Context_i) (\ActionOneHot_i - \EnsemblePolicy[\Round-1](\Context_i))\T \Predictor(\Context_i) } .
\end{align}
The last equality uses the symmetry of $\Predictors$ to introduce the absolute value. The base learning scale constraint, $\frac{1}{\DataSize} \sum_{i=1}^{\DataSize} \frac{\ab{\Reward_i}}{\Propensity_i} \norm{\Predictor(\Context_i)}^2 =  \NormConst$, is automatically satisfied for $\NormConst = \frac{\card{\Actions}}{\DataSize} \sum_{i=1}^{\DataSize} \frac{\ab{\Reward_i}}{\Propensity_i}$ by the fact that $\norm{\Predictor(\Context_i)}^2 = \card{\Actions}$. If there is at least one nonzero reward in the dataset, then $\NormConst > 0$. Thus, minimizing the weighted classification error is equivalent to solving the base learning objective.

The resulting algorithm is given in \cref{alg:boosted_policy_learning_binary_classification}. The optimization problem in line 3---minimizing the weighted classification error---is generally NP-hard \citep{ben-david:jcss03}. However, it is important to remember that the base learner need not fully optimize its learning objective for boosting to be successful; indeed, any predictor for which $\EnsembleWeight_{\Round} \neq 0$ reduces the empirical risk upper bound. In fact, for binary classification, we can relate this condition to a weak learning property.

Let
\begin{equation}
    \NormImportanceWeight_{i, \Action} \defeq \frac{
        \ImportanceWeight_{i, \Action}
    }{
        \sum_{i=1}^{\DataSize} \sum_{\Action \in \Actions} \ImportanceWeight_{i, \Action}
    } ,
\end{equation}
denote a \emph{normalized} weight, which defines an empirical distribution over the data passed to the base learner. Let
\begin{equation}
    \ErrorRate_{\Round}(\Predictor)
    \defeq \sum_{i=1}^{\DataSize} \sum_{\Action \in \Actions} \NormImportanceWeight_{i, \Action} \1\{ \PseudoLabel_{i, \Action} \neq \Predictor(\Context_i, \Action) \} 
\label{eq:weighted_error_rate}
\end{equation}
denote the \emph{error rate} under this distribution, and note that it is proportional to the weighted classification error. Recalling the definition of $\EnsembleWeight_{\Round}$ from \cref{eq:optimal_ensemble_weight}, note that the numerator determines the sign of the expression, and
\begin{align}
    \frac{2}{\DataSize} \sum_{i=1}^{\DataSize} \frac{\Reward_i}{\Propensity_i} \EnsemblePolicy[\Round-1](\Action_i \| \Context_i) (\ActionOneHot_i - \EnsemblePolicy[\Round-1](\Context_i))\T \Predictor_{\Round}(\Context_i)
    &= \frac{2}{\DataSize} \sum_{i=1}^{\DataSize} \sum_{\Action \in \Actions} \ImportanceWeight_{i, \Action} \PseudoLabel_{i, \Action} \Predictor_{\Round}(\Context_i, \Action) \\
    &\propto \sum_{i=1}^{\DataSize} \sum_{\Action \in \Actions} \NormImportanceWeight_{i, \Action} \PseudoLabel_{i, \Action} \Predictor_{\Round}(\Context_i, \Action) \\
    &= \sum_{i=1}^{\DataSize} \sum_{\Action \in \Actions} \NormImportanceWeight_{i, \Action} ( 1 - 2 \, \1\{ \PseudoLabel_{i, \Action} \neq \Predictor_{\Round}(\Context_i, \Action) \} ) \\
    &= 1 - 2 \ErrorRate_{\Round}(\Predictor_{\Round}) .
\end{align}
Thus, $\EnsembleWeight_{\Round}$ is positive when $\ErrorRate_{\Round}(\Predictor_{\Round}) < 1/2$, negative when $\ErrorRate_{\Round}(\Predictor_{\Round}) > 1/2$, and zero when $\ErrorRate_{\Round}(\Predictor_{\Round}) = 1/2$. Recall that boosting can proceed as long as $\EnsembleWeight_{\Round} \neq 0$; meaning, as long as the base learner can produce a classifier that performs better than random guessing under the weighted distribution. This is the weak learning condition for binary classification base learners. If the base learner \emph{always} satisfies this condition, and there exists a constant, $\Advantage \in [0, 1/2)$, such that $\ErrorRate_{\Round}(\Predictor_{\Round}) < \Advantage$ at every round, then (appealing to \cref{th:emp_risk_bound}) the excess empirical risk decays exponentially fast.

\begin{algorithm}
    \caption{Boosted Off-Policy Learning via Binary Classification}
    \label{alg:boosted_policy_learning_binary_classification}
    \begin{algorithmic}[1]
    \Require{symmetric, $\TwoClass$-valued class, $\Predictors$; learning algorithm for weighted binary classification; rounds, $\Rounds \geq 1$}
    \State $\Ensemble[0] \gets 0$
    \For{$\Round = 1, \dots, \Rounds$}
        \State $\Predictor_{\Round} \gets \argmin_{\Predictor \in \Predictors}
            \sum_{i=1}^{\DataSize} \sum_{\Action \in \Actions} \ImportanceWeight_{i, \Action} \1\{ \PseudoLabel_{i, \Action} \neq \Predictor(\Context_i, \Action) \}$
        \Comment{weighted binary classification}
        \Statex \quad\quad\quad with ~
        $\ImportanceWeight_{i, \Action} \gets \ab{
            \frac{\Reward_i}{\Propensity_i} \EnsemblePolicy[\Round-1](\Action_i \| \Context_i) ( \1\{ \Action = \Action_i \} - \EnsemblePolicy[\Round-1](\Action \| \Context_i) )
        }$
        \Statex \quad\quad\quad \, and ~
        $\PseudoLabel_{i, \Action} \gets \sgn(\Reward_i) ( 2 \, \1\{ \Action = \Action_i \} - 1 )$
        \State $\EnsembleWeight_{\Round} \gets \frac{
            2 \sum_{i=1}^{\DataSize} \frac{\Reward_i}{\Propensity_i} \EnsemblePolicy[\Round-1](\Action_i \| \Context_i) (\ActionOneHot_i - \EnsemblePolicy[\Round-1](\Context_i))\T \Predictor_{\Round}(\Context_i)
        }{
            \sum_{i=1}^{\DataSize} \frac{\ab{\Reward_i}}{\Propensity_i} \norm{\Predictor_{\Round}(\Context_i)}^2
        }$
        \State $\Ensemble[\Round] \gets \Ensemble[\Round-1] + \EnsembleWeight_{\Round} \Predictor_{\Round}$
    \EndFor
    \end{algorithmic}
\end{algorithm}

\section{REGULARIZATION}
\label{sec:regularization}

Recall that our original goal is to find a policy with low \emph{expected} risk. Though the IPS estimator is unbiased, optimizing the empirical risk can sometimes lead to overfitting---e.g., if the class of policies is very rich. Accordingly, it is common to add some form of regularization to the learning objective, to penalize model complexity. In our boosting framework, we can assume an ensemble regularizer that decomposes as a sum of base regularizers, $\Regularizer(\Ensemble) = \sum_{\Round=1}^{\Rounds} \Regularizer_{\Round}(\Predictor_{\Round})$. Then, regularization can be applied at the base learner by adding $\Regularizer_{\Round}(\Predictor)$ to line 3 of \cref{alg:boosted_policy_learning,alg:surrogate_boosted_policy_learning}.

Recognizing that the ensemble's complexity is partially determined by the magnitudes of the ensemble weights, we could also choose to regularize the ensemble weights \citep{duchi:icml09}. Note that if we penalize the squared norm of the ensemble weights, adding $\RegParam \sum_{\Round=1}^{\Rounds} \EnsembleWeight_{\Round}^2$ to the learning objective (\cref{eq:policy_learning_objective}), then the closed-form expression for the ensemble weights (using our smoothness-based derivation) would be inversely proportional to the regularization parameter, $\RegParam$. From the perspective of functional gradient descent, $\RegParam > 0$ has the effect of \emph{shrinking} the learning rate, with larger values leading to smaller step sizes.

\section{DATASET DETAILS}
\label{sec:dataset_details}

Covertype is a multiclass classification dataset, consisting of $581{,}012$ records with $54$ binary, ordinal and real-valued features. The task is to predict one of $7$ classes of ground cover for each record. Fashion-MNIST is a multiclass image classification dataset, consisting of $70{,}000$ grayscale images from $10$ categories of apparel and accessories. We extract features from each image by flattening the $(28 \times 28)$-pixel grid to a $784$-dimensional vector. The Scene dataset is a multilabel image classification dataset, consisting of $2{,}407$ records with $294$ numeric features, derived from the images' spatial color moments in LUV space. The task is to determine which of $6$ settings are depicted in the image. Lastly, TMC2007-500 is a multilabel document classification dataset, consisting of $28{,}596$ airplane failure reports. Each record consists of $500$ binary features, indicating the presence or absence of the $500$ most frequent words in the corpus. The task is to determine which of $22$ types of problems are described in the document.

\subsection{Feature Scaling}
\label{sec:feature_scaling}

As feature ranges are important to deep learning (but not boosted tree ensembles), for baselines DRR and BanditNet, we performed some feature scaling. In Covertype, we standardized the ordinal and real-valued features; and in Fashion-MNIST, we normalized the pixel intensities to $[0, 1]$. The Scene and TMC2007-500 datasets' features are already constrained to $[0, 1]$, so no additional feature scaling is necessary.

\subsection{Data Splits}
\label{sec:data_splits}

Of the four datasets, only Fashion-MNIST has a standard training/testing split ($60{,}000$ training; $10{,}000$ testing), which we preserve. We withold a random $10\%$ of the training data for validation. For Covertype, we use a random $80\%/10\%/10\%$ training/validation/testing split. For Scene, we first perform a random $80\%/20\%$ split of data into training and testing partitions, then another $80\%/20\%$ split of the training data into training and validation partitions, ultimately resulting in a $64\%/16\%/20\%$ training/validation/testing split. We perform the same procedure for TMC2007-500 using $90\%/10\%$ splits, resulting in a $81\%/9\%/10\%$ training/validation/testing split. The final dataset sizes are summarized in \cref{tab:dataset_details}.

Note that all of these splits are prior to the supervised-to-bandit conversion (described below), so we have full information for evaluating performance metrics.

\begin{table}[t]
    \caption{Details of the datasets used in our experiments. \emph{Task-types} ``MC'' and ``ML'' denote multiclass and multilabel, respectively. \emph{Logging Training} refers to the number of examples used to train the logging policy, and \emph{Target Training} is the number of simulated bandit feedback examples used to train the target policy.}
    \centering
    \begin{tabular}{lccccccc}
    \toprule
    \bf{Dataset} & \bf{Task-type} & \bf{Features} & \bf{Classes} & \bf{Logging Training} & \bf{Target Training} & \bf{Validation} & \bf{Testing} \\
    \midrule
    \bf{Covertype} & MC & $54$ & $7$ & $46{,}481$ & $418{,}329$ & $58{,}101$ & $58{,}101$ \\
    \bf{Fashion-MNIST} & MC & $784$ & $10$ & $5{,}400$ & $48{,}600$ & $6{,}000$ & $10{,}000$ \\
    \bf{Scene} & ML & $294$ & $6$ & $154$ & $1{,}387$ & $385$ & $481$ \\
    \bf{TMC2007-500} & ML & $500$ & $22$ & $2{,}316$ & $20{,}846$ & $2{,}574$ & $2{,}860$ \\
    \bottomrule
    \end{tabular}
    \label{tab:dataset_details}
\end{table}

\subsection{Supervised-to-Bandit Conversion}
\label{sec:supervised_to_bandit}

To simulate logged bandit feedback from supervised datasets, we use the procedure proposed by \citet{beygelzimer:kdd09}, which has become standard. We start by randomly sampling $10\%$ of the training examples (without replacement) to train a softmax logging policy using supervised learning---in this case, multinomial logistic regression. To avoid overfitting such a small dataset---which could cause the logging policy to become overly confident in certain contexts---we turn up the regularization and employ early stopping.  We then use the logging policy to sample a label (i.e., action) for each remaining training example. In some cases, we enforce a minimum probability for each action by mixing the softmax probabilities with $\epsilon$-greedy exploration. For each sampled label, we record its propensity (under the logging policy) and corresponding reward.

We repeat this procedure $10$ times, using $10$ random splits of the training data, to generate $10$ datasets of logged contexts, actions, propensities and rewards.

\subsection{Task Reward Structure}
\label{sec:task_reward_structure}

For the multilabel datasets, Scene and TMC2007, we assign a full reward of one if the selected action matches one of the true labels, and zero otherwise. For the multiclass datasets, Fashion-MNIST and Covertype, we assign full reward if the true label was selected, but give partial credit if the selected action belongs to a ``near miss'' category. For instance, in Fashion-MNIST, selecting \emph{T-shirt} instead of \emph{shirt} yields a reward of $1/4$. And in Covertype, partial credit is given when the selected action predicts a class of vegetation from the same genus as the actual type; e.g., predicting \emph{aspen} instead of \emph{cottonwood} tree yields a reward of $1/4$.

We define the partial credit class-groupings for Fashion-MNIST as follows:
\begin{itemize}
  \item $\textit{Outerwear} = \{ \textit{coat}, \, \textit{pullover} \}$
  \item $\textit{Shirts} = \{ \textit{T-shirt/top}, \, \textit{shirt} \}$
  \item $\textit{Footwear} = \{ \textit{sandal}, \, \textit{sneaker}, \, \textit{ankle-boot} \}$
\end{itemize}
The rest of the classes (`\textit{trouser}', `\textit{dress}', and `\textit{bag}'), are left as singleton groups with no possibility for partial credit.

Our partial credit groupings for Covertype are:
\begin{itemize}
  \item $\textit{Firs} = \{ \textit{Spruce/Fir}, \, \textit{Douglas-fir} \}$
  \item $\textit{Pines} = \{ \textit{Lodgepole}, \, \textit{Ponderosa} \}$
  \item $\textit{Populus} = \{ \textit{Cottonwood/Willow}, \, \textit{Aspen} \}$
\end{itemize}
The remaining class, `\textit{Krummholz}', is placed in a singleton group.

\subsection{Representation of Contextualized Actions}
\label{sec:contextualized_actions}

For all datasets and methods (except for the logging policy), we construct feature vectors for actions by augmenting the original features in the dataset with an encoding of each action. We accomplish this by concatenating the original features, $\vec{x}$ with a one-hot vector, $\ActionOneHot$, identifying each action, $\Action$; that is $(\vec{x}, \Action) \mapsto [\vec{x}; \ActionOneHot]$, for all $\Action \in \Actions$. This results in $\card{\Actions}$ feature vectors for each example in the dataset. A policy therefore scores all actions (using its predictor or ensemble) and then uses the scores to select an action (either by softmax sampling or argmax).

\section{ALGORITHM IMPLEMENTATION DETAILS}
\label{sec:algo_impl_details}

Following are some relevant implementation details of the algorithms compared in \cref{sec:experiments}.

\begin{itemize}
    \item \textbf{BRR-gb} uses our own implementation of ``vanilla" gradient boosting, with XGBoost as the base learner, configured to fit a single regression tree using the ``exact" splitting algorithm.
    \item \textbf{BRR-xgb} uses XGBoost's implementation of gradient boosted regression trees, with the ``exact" splitting algorithm.
    \item \textbf{DRR} and \textbf{BanditNet} are implemented in MXNet \citep{mxnet}. Each neural network starts with a hidden layer of a given width. For each successive hidden layer (up to the given depth), the width is halved, unless a minimum width of $32$ is reached. Batch normalization \citep{ioffe:icml15} and ReLU activations are used for all hidden layers. For training, we use AdaGrad \citep{duchi:jmlr11} with early stopping on the training reward (estimated via self-normalized IPS \citep{swaminathan:nips15}) and a ``patience" of $10$ epochs.
    \item \textbf{BOPL} and \textbf{BOPL-S} use our own implementations. The \textbf{-regr} variants use XGBoost as the base learner, configured to fit a single regression tree using the ``exact" splitting algorithm. The \textbf{-class} variants use Scikit-learn's \citep{scikit-learn} decision tree classifier as the base learner. We use early stopping whenever the gradient, base predictions or ensemble weight are less than $10^{-10}$ in magnitude.
\end{itemize}

It is important to note that fitting a single regression tree in XGBoost is equivalent (modulo optimizations) to fitting one in a dedicated tree learner (such as Scikit-learn's). This equivalence is unique to the squared error loss, since XGBoost's second-order Taylor approximation is, in this case, exact.

All hyperparameters are tuned via random search and evaluated against held-out validation data for each dataset. \cref{tab:hp_ranges} catalogs the hyperparameters and their associated ranges; \cref{tab:selected_hps} catalogs the values that were selected for each dataset.

\begin{table}[t]
    \caption{Hyperparameter ranges for each method and dataset.}
    \centering
    \begin{tabular}{l l p{.12\textwidth} p{.12\textwidth} p{.12\textwidth} p{.15\textwidth}}
    \toprule
    \bf{Method} & \bf{Parameter} & \bf{Covertype} & \bf{Fashion} & \bf{Scene} & \bf{TMC2007-500} \\
    \midrule
    \bf{BRR-gb}
        & rounds & $[200, 600]$ & $[100, 500]$ & $[100, 1000],$ & $[100, 300]$ \\
        & max\_depth & $[6, 25]$ & $[4, 20]$ & $[5, 14]$ & $[6, 20]$ \\
        & min\_child\_weight & $[2, 400]$ & $[2, 640]$ & $[5, 120]$ & $[2, 200]$ \\
        & reg\_lambda & $[0]$ & $[0]$ & $[0, .05]$ & $[0]$ \\
        & learning\_rate & $[0.01, 0.1]$ & $[0.1, 2]$ & $[0.0001, 0.3]$ & $[0.01, 0.1]$ \\
    \midrule
    \bf{BRR-xgb}
        & rounds (n\_estimators) & $[400, 600]$ & $[100, 500]$ & $[100, 2000]$ & $[100, 300]$ \\
        & max\_depth & $[6, 25]$ & $[10, 20]$ & $[3, 10]$ & $[6, 20]$ \\
        & min\_child\_weight & $[2, 200]$ & $[40, 320]$ & $[5, 120]$ & $[2, 200]$ \\
        & reg\_lambda & $[0, 1.0]$ & $[0, 0.1]$ & $[0, 0.05]$ & $[0, 1.0]$ \\
        & learning\_rate & $[0.01, 0.1]$ & $[0.05, 0.1]$ & $[0.001, 0.1]$ & $[0.01, 0.1]$ \\
    \midrule
    \bf{DRR}
        & num\_hidden\_layers & $[1, 8]$ & $[1, 8]$ & $[1, 8]$ & $[1, 8]$ \\
        & first\_layer\_width & $[32, 256]$ & $[32, 512]$ & $[32, 512]$ & $[32, 512]$ \\
        & dropout & $[0.1, 0.2]$ & $[0.1, 0.2]$ & $[0.1, 0.2]$ & $[0.1, 0.2]$ \\
        & weight\_decay & $[0, 0.001]$ & $[0, 0.001]$ & $[0, 0.001]$ & $[0, 0.001]$ \\
        & learning\_rate & $[0.01, 0.1]$ & $[0.01, 0.1]$ & $[0.01, 0.1]$ & $[0.01, 0.1]$ \\
        & epochs & $[116]$ & $[1000]$ & $[1000]$ & $[1000]$ \\
        & batch\_size & $[1000]$ & $[100]$ & $[100]$ & $[100]$ \\
    \midrule
    \bf{BanditNet}
        & num\_hidden\_layers & $[1, 8]$ & $[1, 8]$ & $[1, 8]$ & $[1, 8]$ \\
        & first\_layer\_width & $[32, 256]$ & $[32, 512]$ & $[32, 512]$ & $[32, 512]$ \\
        & dropout & $[0.1, 0.2]$ & $[0.1, 0.2]$ & $[0.1, 0.2]$ & $[0.1, 0.2]$ \\
        & weight\_decay & $[0, 0.001]$ & $[0, 0.001]$ & $[0, 0.001]$ & $[0, 0.001]$ \\
        & learning\_rate & $[0.01, 0.1]$ & $[0.01, 0.1]$ & $[0.01, 0.1]$ & $[0.01, 0.1]$ \\
        & epochs & $[116]$ & $[1000]$ & $[1000]$ & $[1000]$ \\
        & batch\_size & $[1000]$ & $[100]$ & $[100]$ & $[100]$ \\
        & reward\_translation & $[-0.5 0]$ & $[-0.5 0]$ & $[-0.6 0]$ & $[-0.5 0]$ \\
    \midrule
    \bf{BOPL-regr}
        & rounds & $[200, 600]$ & $[150, 250]$ & $[100, 1000]$ & $[100, 300]$ \\
        & max\_depth & $[10, 25]$ & $[10, 20]$ & $[5, 14]$ & $[6, 20]$ \\
        & min\_child\_weight & $[2, 400]$ & $[100, 640]$ & $[5, 120]$ & $[2, 200]$ \\
        & reg\_lambda & $[0]$ & $[0]$ & $[0, 0.05]$ & $[0]$ \\
        & reward\_translation & $[-0.4, -0.3]$ & $[-0.5, -0.3]$ & $[-0.6 0]$ & $[-0.5, -0.3]$ \\
    \midrule
    \bf{BOPL-class}
        & rounds & $[300]$ & $[150, 250]$ & $[500, 1000]$ & $[200]$ \\
        & max\_depth & $[15, 25]$ & $[15, 25]$ & $[6, 14]$ & $[15, 20]$ \\
        & min\_child\_weight & $[50, 100]$ & $[10, 50]$ & $[60, 240]$ & $[10, 50]$ \\
        & reward\_translation & $[-0.2]$ & $[-0.2]$ & $[-0.4, -0.2]$ & $[-0.15, -0.1]$ \\
    \midrule
    \bf{BOPL-S-regr}
        & rounds & $[400, 600]$ & $[100, 250]$ & $[100, 1000]$ & $[100, 300]$ \\
        & max\_depth & $[10, 25]$ & $[10, 20]$ & $[5, 14]$ & $[6, 20]$ \\
        & min\_child\_weight & $[2, 400]$ & $[100, 640]$ & $[5, 120]$ & $[2, 200]$ \\
        & reg\_lambda & $[0]$ & $[0]$ & $[0, 0.05]$ & $[0]$ \\
        & reward\_translation & $[-0.4, -0.3]$ & $[-0.5, -0.3]$ & $[-0.6 0]$ & $[-0.5, -0.3]$ \\
    \midrule
    \bf{BOPL-S-class}
        & rounds & $[300]$ & $[150, 250]$ & $[500, 1000]$ & $[200]$ \\
        & max\_depth & $[15, 25]$ & $[15, 25]$ & $[6, 14]$ & $[15, 20]$ \\
        & min\_child\_weight & $[50, 100]$ & $[10, 50]$ & $[60, 240]$ & $[10, 50]$ \\
        & reward\_translation & $[-0.2]$ & $[-0.2]$ & $[-0.4, -0.2]$ & $[-0.15, -0.1]$ \\
    \bottomrule
    \end{tabular}
    \label{tab:hp_ranges}
\end{table}

\begin{table}[t]
    \caption{Selected hyperparameters for each method and dataset.}
    \centering
    \begin{tabular}{l l p{.12\textwidth} p{.12\textwidth} p{.12\textwidth} p{.15\textwidth}}
    \toprule
    \bf{Method} & \bf{Parameter} & \bf{Covertype} & \bf{Fashion} & \bf{Scene} & \bf{TMC2007-500} \\
    \midrule
    \bf{BRR-gb}
        & rounds & $600$ & $250$ & $1000$ & $100$ \\
        & max\_depth & $20$ & $10$ & $12$ & $10$ \\
        & min\_child\_weight & $2$ & $100$ & $20$ & $2$ \\
        & reg\_lambda & $0$ & $0$ & $0.01$ & $0$ \\
        & learning\_rate & $0.1$ & $0.1$ & $0.01$ & $0.1$ \\
    \midrule
    \bf{BRR-xgb}
        & rounds (n\_estimators) & $400$ & $500$ & $2000$ & $200$ \\
        & max\_depth & $25$ & $15$ & $8$ & $10$ \\
        & min\_child\_weight & $10$ & $100$ & $20$ & $2$ \\
        & reg\_lambda & $1$ & $0$ & $0.05$ & $1$ \\
        & learning\_rate & $0.1$ & $0.1$ & $0.005$ & $0.1$ \\
    \midrule
    \bf{DRR}
        & num\_hidden\_layers & $4$ & $4$ & $2$ & $1$ \\
        & first\_layer\_width & $256$ & $512$ & $512$ & $256$ \\
        & dropout & $0.1$ & $0.1$ & $0.2$ & $0.1$ \\
        & weight\_decay & $0$ & $10^{-6}$ & $10^{-5}$ & $10^{-6}$ \\
        & learning\_rate & $0.01$ & $0.01$ & $0.01$ & $0.01$ \\
        & epochs & $116$ & $1000$ & $1000$ & $1000$ \\
        & early\_stopping\_patience & $10$ & $10$ & $10$ & $10$ \\
        & batch\_size & $1000$ & $100$ & $100$ & $100$ \\
    \midrule
    \bf{BanditNet}
        & num\_hidden\_layers & $4$ & $4$ & $4$ & $2$ \\
        & first\_layer\_width & $256$ & $512$ & $512$ & $512$ \\
        & dropout & $0.1$ & $0.1$ & $0.1$ & $0.1$ \\
        & weight\_decay & $10^{-4}$ & $10^{-5}$ & $10^{-4}$ & $10^{-3}$ \\
        & learning\_rate & $0.01$ & $0.01$ & $0.01$ & $0.01$ \\
        & epochs & $116$ & $1000$ & $1000$ & $1000$ \\
        & early\_stopping\_patience & $10$ & $10$ & $10$ & $10$ \\
        & batch\_size & $1000$ & $100$ & $100$ & $100$ \\
        & reward\_translation & $-0.4$ & $-0.42$ & $-0.2$ & $-0.2$ \\
    \midrule
    \bf{BOPL-regr}
        & rounds & $600$ & $250$ & $500$ & $300$ \\
        & max\_depth & $25$ & $20$ & $12$ & $20$ \\
        & min\_child\_weight & $200$ & $200$ & $60$ & $200$ \\
        & reg\_lambda & $0$ & $0$ & $0.01$ & $0$ \\
        & reward\_translation & $-0.4$ & $-0.41$ & $-0.2$ & $-0.3$ \\
    \midrule
    \bf{BOPL-class}
        & rounds & $300$ & $150$ & $1000$ & $200$ \\
        & max\_depth & $25$ & $25$ & $14$ & $20$ \\
        & min\_child\_weight & $50$ & $50$ & $60$ & $50$ \\
        & reward\_translation & $-0.2$ & $-0.2$ & $-0.35$ & $-0.1$ \\
    \midrule
    \bf{BOPL-S-regr}
        & rounds & $600$ & $250$ & $100$ & $300$ \\
        & max\_depth & $25$ & $20$ & $10$ & $20$ \\
        & min\_child\_weight & $200$ & $200$ & $60$ & $200$ \\
        & reg\_lambda & $0$ & $0$ & $0$ & $0$ \\
        & reward\_translation & $-0.4$ & $-0.4$ & $-0.2$ & $-0.3$ \\
    \midrule
    \bf{BOPL-S-class}
        & rounds & $300$ & $150$ & $1000$ & $200$ \\
        & max\_depth & $25$ & $25$ & $14$ & $20$ \\
        & min\_child\_weight & $50$ & $50$ & $120$ & $50$ \\
        & reward\_translation & $-0.2$ & $-0.2$ & $-0.2$ & $-0.1$ \\
    \bottomrule
    \end{tabular}
    \label{tab:selected_hps}
\end{table}

\end{document}